\documentclass[a4paper,12pt]{article}
\usepackage[a4paper]{geometry}

\usepackage{amsmath,amsthm,mathtools}
\usepackage{amssymb}
\usepackage{algorithmic, algorithm}
\usepackage{xspace,xcolor}
\usepackage{dsfont}
\usepackage{url}
\usepackage{graphicx}
            
\def\tr#1{\lfloor #1\rfloor}
\def\ve{\varepsilon}
\def\cl#1{\lceil #1\rceil}

\def\lpa#1{\bigl({#1}\bigr)}
\def\Lpa#1{\Bigl({#1}\Bigr)}
\def\dd#1{{\,\rm d}#1}

\newtheorem{theorem}{Theorem}

\newtheorem{corollary}[theorem]{Corollary}
\newtheorem{definition}[theorem]{Definition}
\newtheorem{lemma}[theorem]{Lemma}

\newcommand{\filt}{\mathcal{F}_t}
\newcommand{\filtzero}{\mathcal{F}_0}

\newcommand{\R}{{\mathds{R}}}

\newcommand{\indic}[1]{\mathds{1}_{#1}}

\renewenvironment{proof}%
{\begin{trivlist}\item\textbf{Proof.}}%
{\hspace*{\fill}$\Box$\end{trivlist}}

\newenvironment{proofof}[1]%
{\begin{trivlist}\item\textbf{Proof of #1.}}%
{\hspace*{\fill}$\Box$\end{trivlist}}

\newcommand{\expect}[1]{\mathord{E}\hspace{0.3ex}\mathord{\left(#1\right)}}

\newcommand{\prob}[1]{\mathord{\Prob}\mathord{\left(#1\right)}}

\newcommand{\xmin}{x_{\mathrm{min}}}

\newcommand{\OneMax}{\textsc{OneMax}\xspace}
\newcommand{\onemax}{\OneMax}
\newcommand{\om}{\OneMax}

\newcommand{\ie}{i.\,e.\xspace}
\newcommand{\eg}{e.\,g.\xspace}

\renewcommand{\epsilon}{\varepsilon}
\DeclareMathOperator{\Prob}{Pr}
\newcommand{\ea}{(1+1)~EA\xspace}
\newcommand{\oneoneea}{\ea}

\allowdisplaybreaks[4]

\author{\small Hsien-Kuei Hwang\\
\small Institute of Statistical Science\\
\small Academia Sinica\\
\small Taipei 115\\
\small Taiwan
\and 
\small Carsten Witt\\
\small Technical University of Denmark\\
\small Kgs. Lyngby\\
\small Denmark
}

\title{Sharp Bounds on the Runtime of the (1+1) EA via Drift Analysis and
Analytic Combinatorial Tools}

\begin{document}

\maketitle

\begin{abstract}
The expected running time of the classical \ea on the \onemax benchmark function has recently been determined  
by Hwang et al.~(2018) up to additive errors of $O((\log n)/n)$. 
The same approach proposed there also leads to a 
full asymptotic expansion with errors of the form $O(n^{-K}\log n)$ 
for any $K>0$.
This precise result  is obtained  by matched asymptotics with rigorous error analysis 
(or by solving asymptotically the underlying recurrences via inductive approximation arguments), ideas
 radically different from well-established techniques for the running time analysis 
of evolutionary computation such as drift analysis. This paper revisits drift analysis for the \ea on \onemax and obtains 
that the expected running time $\expect{T}$, starting from $\cl{n/2}$
 one-bits, is determined by the sum of inverse drifts up to logarithmic error terms, more precisely 
\[
\sum_{k=1}^{\tr{n/2}}\frac{1}{\Delta(k)} - c_1\log n \le \expect{T} \le \sum_{k=1}^{\tr{n/2}}\frac{1}{\Delta(k)} - c_2\log n
\]
where $\Delta(k)$ is the drift (expected increase of the number of one-bits from the state of~$n-k$ ones) and $c_1,c_2 \, >0$ are 
explicitly computed constants. This improves the previous asymptotic error known for the sum of inverse drifts 
from $\tilde{O}(n^{2/3})$ to a logarithmic error and gives for the first time a non-asymptotic error bound. Using 
standard asymptotic techniques, the difference between $\expect{T}$ and the sum of inverse drifts 
is found to be $(e/2)\log n+O(1)$.
\end{abstract}

\section{Introduction}

The runtime analysis of randomized search heuristics on simple, well-structured benchmark problems
has triggered the development of analytical tools for understanding the complexity and considerably contributed 
to their theoretical foundations. This paper is concerned with the objective function
 $\onemax(x_1,\dots,x_n)=x_1+\dots+x_n$, 
the arguably most fundamental 
theoretical benchmark problem in discrete search spaces and the \ea, probably the most fundamental search heuristic in the 
theoretical runtime analysis (see Algorithm~\ref{alg:oneoneea}).

Already the earliest analysis of the \ea \cite{Muhlenbein92} showed that the \ea optimizes 
\onemax in an expected time of $O(n\log n)$, where time corresponds to the number of
 iterations. The early interest and attempts in obtaining more precise description of the runtime complexity were summarized in Garnier et al.'s fine paper 
\cite{Garnier1999} with very strong approximation results claimed. 
On the other hand,
 it follows from the analyses in 
\cite{djwea02} that the 
expected time is bounded from above by $en H_n \le en(\log n+1)$, where $H_n=\sum_{j=1}^n 1/j$ 
denotes the $n$-th harmonic number and $\log n$ the \emph{natural} logarithm. Lower bounds of the kind 
$\Omega(n\log n)$ that hold for the much 
larger class of functions with a unique optimum \cite{djwea02} showed that the 
results were at least asymptotically tight.

From the beginning of this decade, finer analyses of the expected runtimes 
have gained increasing attention. 
Precise expressions for the runtime, dependent not only on the search 
space dimension~$n$ but also on parameters such as 
the mutation rate, are vital to optimize parameter settings \cite{WittCPC13} and to compare different 
algorithms whose runtime only differs in lower-order terms \cite{DoerrDYGECCO16}. 

With respect to \onemax, the first lower bound 
that explicitly states the leading coefficient $e$ in an expression of the type  
${(1-o(1))en\log n}$ was independently derived by Doerr, Fouz and Witt~\cite{DFW10}  and 
Sudholt~\cite{SudholtLowerFitnessLevel} (in the finer form $en\log n-2n\log\log n$) 
using the techniques of drift analysis and 
fitness levels, respectively. The lower-order term was sharpened to a linear term $\Omega(n)$ 
in~\cite{DFW11}, and an explicit bound for the coefficient of this linear term, 
was given by Lehre and Witt~\cite{LehreWittISAAC14}, who proved the lower bound 
$en\log n-7.81791n-O(\log n)$. The main tool to derive these results relies on 
increasingly refined drift theorems, most notably on variable drift analysis. At roughly 
the same time, Hwang et al.~\cite{HwangArxiv14} presented a drastically refined analysis, which 
determines the expected runtime of the \ea on \om up to terms of order $O((\log n)/n)$: the exact 
expression given is
\begin{equation}
en \log n - C_1 n + (e/2)\log n + C_2 +O((\log n)/n),
\label{eq:exact-hwang}
\end{equation}
where $C_1=1.89254\dots$ and $C_2=0.59789875\dots$ are explicitly computable constants; see also \cite{HwangEVCO18} 
for the journal version 
and \url{http://140.109.74.92/hk/?p=840}
for the web version with a full asymptotic expansion.
 To obtain these precise results, techniques fundamentally different from 
drift analysis and other established methods for the runtime analysis were used, namely matched asymptotics 
with rigorous error analysis. 
In addition to the expected runtime, the asymptotic variance as well as the limiting distribution 
are also worked out there by similar approaches.

While the expression for the asymptotic 
expected runtime in~\eqref{eq:exact-hwang} represents the best of its kind, 
it also raises important open questions.
First, from a more didactical and methodological point of view, 
one may look for a more elementary derivation of the formula \eqref{eq:exact-hwang}, at least 
with respect to the linear term $-C_1 n$. 
Note that one  
 can analyze the related search heuristic RLS, which 
flips exactly one bit per iteration, on \om 
exactly and without any asymptotic terms (see \cite{DoerrDoerrAlgo16}), at least 
if it is initialized deterministically with $\cl{n/2}$ one-bits. The 
expression of the expected runtime equals $nH_{\tr{n/2}}$ then and is accompanied by 
an intuitive proof appealing to the coupon collector theorem. For a uniform initialization, 
the analysis become more involved but still an extremely precise result (coming with 
an asymptotic term, though) exists: $nH_{\tr{n/2}}-1/2+o(1)$.
 This proof takes only a few pages 
and uses well-known intuitive concepts such as the binomial distribution. The $o(1)$-term 
comes without an explicit error bound, though, and it is not discussed how to refine it.

Second, it would be helpful to confirm that the constant in the $O((\log n)/n)$-term 
is small so that one may call it negligible even for small problem sizes. 
This question may be approached along two different directions: one via an explicit error bound for all~$n$, and the other by combining exact numerical calculations and asymptotic expansions. The former will be realized by the drift analysis presented 
in Sections~\ref{sec:lower}--\ref{sec:upper} of 
this paper; we briefly describe here the latter, which depends on the sample size $n$. If $n$ is large enough, say $n\ge50$, then 
we can use a longer expansion of the form 
\[
    en \log n - C_1 n + \sum_{0\le k\le K}
    \frac{d_k \log n + e_k}{n^k} 
\]
where $K$ is chosen large enough depending on the required tolerance error. In particular, by refining the analysis 
in~\cite{HwangEVCO18},  one has $d_0=\frac12$, $d_1=\frac98$ and $d_2=\frac{31}{16}$ (expressions for $e_k$ being more complex). 
On the other hand, if $n$ is small, one can always compute the exact quantity by the underlying recurrence 
relations without introducing any error. Such an exact calculation can be made efficient even in portable computing devices such as laptops and for $n$ in the hundreds; it is equally helpful in measuring the error introduced when using $K$ terms of the asymptotic expansion.

While different approaches have their own strengths and weaknesses, it is possible to combine them in many cases in discrete probabilities and algorithmics, and obtain results that are often stronger than a single approach can achieve. The fine approximation we work out in this paper represents another testimony to this statement.

\paragraph{Our contribution.} In this paper, we revisit the method of drift analysis 
and obtain 
that the expected runtime $\expect{T}$ of the \ea on \om, started from $\cl{n/2}$ ones,  
is approximated by the sum of inverse drifts up to logarithmic error terms, more precisely
\[
\sum_{k=1}^{\tr{n/2}}\frac{1}{\Delta(k)} - c_1\log n \le \expect{T} \le \sum_{k=1}^{\tr{n/2}}\frac{1}{\Delta(k)} - c_2\log n,
\]
where $\Delta(k)$ is the drift (expected increase of the number of one-bits from the state of~$n-k$ ones) and $c_1,c_2\,>0$ are 
explicitly computed constants. This gives not only an intuitive approximation of the expected runtime 
via inverse drifts   
but for the first time explicit error bounds. Closest to our results, Gie{\ss}en and Witt~\cite{GiessenWittAlgo18} used 
new variants of variable drift analysis and 
showed for the more general class of (1+$\lambda$)~EAs that the expected runtime is characterized 
by the sum of inverse drifts up to an additive error of $\tilde{O}(n^{2/3})$ --- we improve further this error 
term to $c\log n$ for an explicit constant~$c>0$. To prove our results, we use elementary techniques 
and additive drift analysis as the only tool for the treatment of stochastic processes. At the same time, 
we obtain new drift theorems dealing with error bounds in variable drift analysis that may be of independent interest.
The assumption of a fixed starting point for the \ea 
only introduces a difference in $O(1)$ compared to the expected 
runtime with a uniform initialization \cite{DoerrDoerrAlgo16}.

Finally, from the sum expression of $\Delta(k)$, we prove, by standard asymptotic methods 
(generating functions and the Euler-Maclaurin formula), that
the expected runtime of the \ea on \om equals
\[
\sum_{k=1}^{\tr{n/2}} \frac{1}{\Delta(k)} - \frac{e}{2}\log n + O(1),
\]
\ie, the sum of inverse drifts overestimates the exact expected time only by an additive term 
of $(e/2)\log n+O(1)$.

This paper is structured as follows. In Section~\ref{sec:preliminaries}, we introduce the 
concrete problem setting and well-known variable drift theorems. We also revisit the 
well-known result that the expected runtime of the \ea on \om is bounded by the sum 
of inverse drifts over the interval $\{1,\dots,X_0\}$, where $X_0$ is the initial number 
of zero-bits. Section~\ref{sec:lower} is concerned with the lower bound $\sum_{k=1}^{\tr{n/2}} 1/\Delta(k) - c_1 \log n$ 
for a constant $c_1>0$, which we prove using a new, self-contained variable drift theorem. Section~\ref{sec:upper} 
complements this result by bounding the expected runtime from above by $\sum_{k=1}^{\tr{n/2}} 1/\Delta(k) - c_2 \log n$, for 
another constant~$c_2>0$, again 
using a novel variable drift theorem. 
The following Section~\ref{sec:exact} then briefly illustrates that drift analysis in principle
allows an alternative proof of  an exact expression of the expected runtime, before we in 
Section~\ref{sec:asymptotic-expansion} 
apply asymptotic techniques to show that the expression 
$\sum_{k=1}^{\tr{n/2}} 1/\Delta(k) - (e/2) \log n$  gives the exact time 
up to additive errors of $O(1)$. 

\section{Preliminaries}
\label{sec:preliminaries}
We consider the classical randomized search heuristic \ea; see   
Algorithm~\ref{alg:oneoneea}, which is intensively 
studied in the theory of randomized search heuristics \cite{AugerDoerrBook,JansenBook}. It creates 
a new search point by flipping each bit of the current search point independently with probability~$1/n$ and 
accepts it if it is not inferior to the previous search point. The algorithm 
is formulated for pseudo-boolean maximization problems but can straightforwardly be applied to 
minimization as well. The analysis of the \ea is a stepping stone 
towards the analysis of more advanced search heuristics, but already this simple framework leads to 
challenging analyses even on very simple problems.  In this paper, we focus exclusively on the simple 
\onemax problem, which can be regarded as a simple hillclimbing task. 

\begin{algorithm}
\caption{\oneoneea}
\label{alg:oneoneea}
\begin{algorithmic}
\STATE $t:=0$.
 \STATE Choose uniformly at random $x_0 \in \{0,1\}^n$.
 \REPEAT 
  \STATE Create $x'$ by flipping each bit in $x_t$ independently with probability $1/n$.
  \STATE $x_{t+1}:=x'$ if $f(x') \ge f(x_t)$, and $x_{t+1}:=x_t$ otherwise. 
  \STATE $t:=t+1$.
 \UNTIL{some stopping criterion is fulfilled.}
\end{algorithmic}
\end{algorithm}

Since the 
\ea is unbiased, \ie, it treats one-bits and zero-bits in the same way \cite{LehreWittAlgorithmica12}, 
all  results in this paper
hold also for the more general Hamming distance minimization 
problem $f_z(x)=n-H(x,z)$, where $z\in\{0,1\}^n$ is arbitrary and 
$H(x,z)$ denotes the Hamming distance of the search points~$x$ and~$z$. We also remark that our forthcoming 
analyses can be generalized to different mutation rates, \ie, a \oneoneea that flips each bit independently 
with probability~$c/n$ for a constant~$c>0$; however, this will not yield new interesting insights. We emphasize 
that we only consider a static mutation probability here -- dynamic schemes, including 
self-adjusting and self-adaptive mutation rates (\eg, \cite{DoerrDYGECCO16,DoerrGWYAlgo19}) must usually be analyzed via 
different techniques.

The \emph{runtime} (synonymously, \emph{optimization time}) is the smallest $t$ such that~$x_t$ is optimal, 
\ie, the random number of iterations 
until sampling an optimum. It corresponds to the number of fitness evaluations (plus $1$ for the initialization) 
until the optimum is found. In this paper, we are exclusively concerned with the expected runtime;  bounds 
on the tail of the runtime of \ea can be found, \eg, in~\cite{LehreWittISAAC14}.

\subsection{Additive Drift}
Our main tool for the runtime analysis of the \ea is \emph{drift analysis}, which is in fact one of the most 
versatile and wide-spread techniques for this purpose \cite{LenglerDriftSurveyArxiv}. Roughly speaking, 
drift analysis translates information about the expected local change of the process (the so-called drift) into 
a global statement about the first hitting time of a target state. Drift analysis, which is well known 
in the theory of stochastic processes \cite{Hajek1982}, was introduced to the 
field of runtime analysis of evolutionary computation by He and Yao~\cite{heyao2001} in the form of an additive drift theorem. 
This theorem was continuously refined and given in different formulations. We present it in a very general 
style, allowing continuous state spaces and non-Markovian processes. As noticed by Lengler~\cite{LenglerDriftSurveyArxiv} and 
Krejca and K{\"o}tzing~\cite{KoetzingKrejcaPPSN2018-drift}, the process may live on a one-sided unbounded state space if 
upper bounds on the expected first hitting time are to be derived. We also integrate both 
variants for upper and lower bounds on expected hitting times in one theorem, sacrificing some generality in 
the second case \cite{KoetzingKrejcaPPSN2018-drift}.

\begin{theorem}
\label{theo:additive-drift}
Let $(X_t)_{t\ge 0}$ be a stochastic process, adapted to a filtration
$\filt$, over some state space $S\subseteq \R^{\ge 0}$, where $0\in S$. 
Let $T:=\min\{t\mid X_t=0\}$ be the 
first hitting time of state~$0$. 

\begin{enumerate}
\item 
If there is some $\delta>0$ such that conditioned on $t<T$ it holds 
that
\[
\expect{X_t-X_{t+1}\mid \filt} \ge \delta,
\]
then
\[
\expect{T\mid \filtzero} \le \frac{X_0}{\delta}.
\]
\item 
If there is some $\delta>0$ such that conditioned on $t<T$ it holds 
that both
\[
\expect{X_t-X_{t+1}\mid \filt} \le \delta,
\]
and $X_t\le b$ for some constant $b> 0$ 
then
\[
\expect{T\mid \filtzero} \ge \frac{X_0}{\delta}.
\]
\end{enumerate}
\end{theorem}

In a nutshell, Theorem~\ref{theo:additive-drift} estimates the first hitting time of the target~$0$ 
by the initial distance divided by the average process towards the target. Clearly, 
if the worst-possible $\delta$ over the state space is very small, then the resulting 
bound on the expected hitting time (in part~1) may overestimate the truth considerably. To obtain 
more precise results, one may transform the actual state space $X_t$ to a new state space 
$g(X_t)$ via a so-called potential (Lyapunov) function $g\colon S\to \R^{\ge 0}$. If the 
drift of the process $g(X_t)$ is similar all over the search space then more precise bounds are obtained. 
This idea of smoothing out the drift over the state space underlies most advanced drift theorems 
such as multiplicative drift \cite{DJWMultiplicativeAlgorithmica} and variable 
drift \cite{Johannsen10}. Since multiplicative drift is a special case of variable drift, we will 
focus exclusively on additive and variable drift in the remainder of this paper.

\subsection{Variable Drift}
The first theorems stating upper bounds on the hitting time
using variable drift go back to \cite{Johannsen10} and \cite{MitavskiyVariable}. These theorems  were subsequently
generalized in~\cite{RoweSudholtChoiceJournal} and~\cite{LehreWittISAAC14}. Similarly to Theorem~\ref{theo:additive-drift}, 
we present a general version allowing non-Markovian processes and unbounded state spaces. We also give 
a self-contained proof.

\begin{theorem}[Variable Drift, upper bound]
\label{theo:variable-upper}
Let $(X_t)_{t\ge 0}$ be a stochastic process, adapted to a filtration
$\filt$, over some state space $S\subseteq \{0\}\cup \R^{\ge \xmin}$, where $\xmin>0$. 
Assume $0\in S$ and define $T:=\min\{t\mid X_t=0\}$. 

Let  $h\colon \R^{\ge \xmin}\to\R^+$ be a monotone increasing function 
and suppose that 
$E(X_t-X_{t+1} \mid \filt) \ge h(X_t)$ conditioned on $t<T$.
 Then it holds 
 that
\[
\expect{T\mid \filtzero} \le
\frac{\xmin}{h(\xmin)} + \int_{\xmin}^{X_0} \frac{1}{h(x)} \,\mathrm{d}x\enspace.
\]
\end{theorem}

\begin{proof}
We will apply Theorem~\ref{theo:additive-drift} (part~1) with respect to the process 
$g(X_t)$, where the 
potential function $g(x)$ be defined by 
\[
g(x) \coloneqq \frac{\xmin}{h(\xmin)} + \int_{\xmin}^x \frac{1}{h(z)} \,\mathrm{d}z.
\]

We note that $g$ is concave since $1/h$ is monotone decreasing by assumption.
Considering  
the drift of $g$, we have
\[
\expect{g(X_t)-g(X_{t+1})\mid \filt} = 
 \int_{\xmin}^{X_t} \frac{1}{h(z)} \,\mathrm{d}z 
 - \expect{\int_{\xmin}^{X_{t+1}} \frac{1}{h(z)} \,\mathrm{d}z\mid \filt}.
\]
By Jensen's inequality, we obtain for $t<T$ 
\[
\expect{g(X_t)-g(X_{t+1})\mid \filt} = 
 \int_{\xmin}^{X_t} \frac{1}{h(z)} \,\mathrm{d}z 
 - \int_{\xmin}^{\expect{X_{t+1}\mid \filt}} \frac{1}{h(z)} \,\mathrm{d}z,
\]
which, since $\expect{X_{t+1}\mid \filt} \le X_t-h(X_t)$, 
is at least 
\[
  \int_{X_t-h(X_t)}^{X_t} \frac{1}{h(z)} \,\mathrm{d}z \ge 
	 \int_{X_t-h(X_t)}^{X_t} \frac{1}{h(X_t)} \,\mathrm{d}z = 1, 
\]
where the inequality used that $h(z)$ in non-decreasing. 
The theorem now follows by Theorem~\ref{theo:additive-drift}, part~1.  
\end{proof}

We remark that we can avoid applying Jensen's inequality in the above proof by splitting
\begin{align*}
\expect{g(X_{t+1})\mid\filt} & = \expect{g(X_{t+1})\indic{X_{t+1}\le X_t}\mid\filt} \\
& \quad + 
\expect{g(X_{t+1})\indic{X_{t+1}>X_t}\mid\filt}
\end{align*}
and estimating $1/h(z)$ 
from above by $h(X_t)$ if $X_{t+1}>X_t$ by taking a change of sign into account \cite{LehreWitt2013arXivPreprint}. 
However, we find that this leads 
to a less easily readable proof. In any case, the variable drift theorem upper bounds the 
expected time to reach state~$0$ because $h(x)$ is non-decreasing by assumption. If $h(x)$ 
was non-increasing, we could conduct an analogous proof to bound $\expect{T}$ from below; however, 
usually the drift of a process increases with the distance from its target.

For discrete search spaces, the variable drift theorem can be simplified (see also \cite{RoweSudholtChoiceJournal}). We present 
the following version for Markov processes on the integers.
\begin{corollary}
\label{cor:discrete-variable}
Let $(X_t)_{t\ge 0}$ be a Markov process on the state space $\{0,\dots,N\}$ for some integer~$N$.
Let  $\Delta\colon \{1,\dots,N\} \to \R^+$ be a monotone increasing function such that
$E(X_t-X_{t+1} \mid X_t=k) \ge \Delta(k)$.
 Then it holds for the first hitting time
$T:=\min\{t\mid X_t=0\}$ that
\[
\expect{T\mid X_0} \le
\sum_{k=1}^{X_0}  \frac{1}{\Delta(k)}.
\]
\end{corollary}

\subsection{First Upper Bound for \om}
\label{sec:triv-upper}

Corollary~\ref{cor:discrete-variable} is ready to use for our scenario of the analysis of the \ea on \om. 
We identify state~$k$ with all search points having $k$ zero-bits (\ie, $n-k$ one-bits), think of the \ea minimizing the number of zero-bits
 and note that state~$0$ is the optimal 
state. If we instantiate the corollary with
\begin{equation}
\Delta(k) \coloneqq \sum_{\ell=1}^k \sum_{j=0}^{\ell} (\ell-j)\binom{k}{\ell}\binom{n-k}{j} \left(\frac{1}{n}\right)^{\ell+j}\left(1-\frac{1}{n}\right)^{n-\ell-j} ,
\label{eq:delta-i}
\end{equation}
where as usual $\binom{a}{b}=0$ if $b<0$ or $b>a$, 
which is the exact expression for the expected decrease in the number of zero-bits from~$k$ such bits, 
then we obtain an upper bound on the runtime of the \ea on \om, started with $X_0$ zero-bits. This result 
is well known  and it can easily be shown that 
\[
\sum_{k=1}^{X_0}  \frac{1}{\Delta(k)} \le enH_{X_0}
\]
since $\Delta(k)\ge \frac{k}{n}(1-1/n)^{n-1}\ge e^{-1}k/n$ by 
considering all steps flipping exactly one bit out of the~$k$ zeros and no other bits. However, no exact closed-form 
expression for $\sum_{k=1}^{X_0}  \frac{1}{\Delta(k)} $ is known in general.

%
%

\section{Lower Bounds}
\label{sec:lower}

\subsection{Variable Drift with Error Bound}

In light of the simple upper bound presented above in Section~\ref{sec:triv-upper}, it is interesting 
to study how tight this bound is. Previous research addressed this question usually by
\begin{itemize}
\item Proving an analytical upper bound on the expected value of $Q_k\coloneqq \sum_{j=1}^{k}  \frac{1}{\Delta(j)}$ 
(for a random starting state~$k$)
\item  Bounding $\expect{T}$ from below by using specific variable drift theorems for lower bounds. 
The sum $Q_k$ did not explicitly show up in these bounds.
\end{itemize}
As a result, this approach estimates the error made by bounding $\expect{T\mid X_0=k}\le Q_k$ only indirectly. 
One notable example is the work by Gie{\ss}en and Witt~\cite{GiessenWittAlgo18}, who prove the nesting 
\[
(1-O(n^{-1/3}\log n)) Q_k \le \expect{T\mid X_0=k} \le Q_k,
\]
which shows that the sum of inverse drifts $Q_k$ represents the expected optimization time of (1+$\lambda$)~EAs 
on \om from state~$k$ 
up to polynomial lower order terms (which would be in the order of $O(n^{2/3}\log^2 n)$ for those starting 
points from which it takes expected time $\Omega(n\log n)$). Interestingly, this result 
was obtained by a new variable drift theorem for lower bounds that can be instantiated with the concrete 
setting of optimizing \om. In this setting, one can identify the sum of inverse drifts $Q_k$ up to lower order terms.

In this section, we follow an even more direct approach to relate $\expect{T\mid X_0=k}$ to~$Q_k$. As already mentioned, 
several variants of variable drift theorems for proving lower bounds on hitting times have been proposed; see again 
\cite{GiessenWittAlgo18} 
for a recent discussion.
The main challenge proving such lower bounds is that the potential function~$g(x)$ proposed in the 
proof of Theorem~\ref{theo:variable-upper} is concave, so Jensen's inequality cannot be used 
to bound the drift of the potential function from above. However, if one can estimate the exact 
drift of the potential function and bound it uniformly from below for all non-optimal states, 
we get a lower bound for the expected first hitting time. We make this 
explicit for discrete search spaces in the following; however, the approach would easily generalize 
to continuous spaces. We restrict ourselves to non-increasing processes for notational convenience 
but note that we could allow $X_{t+1}$ to be greater than~$X_t$ by adjusting the definition of $\eta(k)$ in 
the following theorem slightly.

\begin{theorem}
\label{theo:var-lower-error}[Variable drift, lower bound, with error bound]
Let $(X_t)_{t\ge 0}$ be a non-increasing Markov process on the state space $\{0,\dots,N\}$ for some integer~$N$.
Let  $\Delta\colon \{1,\dots,N\} \to \R^+$ be a function satisfying
$E(X_t-X_{t+1} \mid X_t=k) \le \Delta(k)$ for $k\in\{1,\dots,N\}$. Let 
\[
\eta(k) \coloneqq  \expect{\sum_{j=X_{t+1}+1}^{k} \frac{1}{\Delta(j)}  \bigm| X_t=k}  
\]
and
\[
\eta^* \coloneqq \max_{k=1,\dots,N} \eta(k).
\]
 Then it holds for the first hitting time
$T:=\min\{t\mid X_t=0\}$ that
\[
\expect{T\mid X_0} \ge
\sum_{k=1}^{X_0}  \frac{1}{ \eta^*\Delta(k)} .
\]
\end{theorem}

Hence, $\eta(k)$ is an error bound quantifying the relative error 
incurred by using the sum of inverse drifts as an estimate for the expected first hitting time from 
state~$k$, and $\eta^*$ is the worst case of the $\eta(k)$ over all non-target states.

\begin{proof}
We consider the same potential function $g(k)=\sum_{j=1}^k 1/\Delta(j)$ as in the proof of 
Theorem~\ref{theo:variable-upper} and note that 
its drift at point~$k$ equals
\begin{align*}
\expect{g(k)-g(X_{t+1})\mid X_t=k} & 
=  \sum_{j=1}^{k} \frac{1}{\Delta(k)}- \expect{\sum_{j=1}^{X_{t+1}} \frac{1}{\Delta(j)}  \bigm| X_t=k} \\  
& = \eta(k).
\end{align*}
By the additive drift theorem (Theorem~\ref{theo:additive-drift}, part~2) 
with potential function $g(x)$ and upper bound $\eta^*$ on the drift, the theorem follows.
\end{proof}

We will use the previous variable drift theorem to obtain the following lower bound.
\begin{theorem}
\label{theo:lower}
Let $\expect{T}$ denote the expected optimization time of the \ea on \om, started with 
$\cl{n/2}$ one-bits and let $\Delta(k)$ be the drift of the number of zeros as defined in 
Definition~\eqref{eq:delta-i}. Then 
\[
\expect{T} \ge \sum_{k=1}^{\tr{n/2}} \frac{1}{\Delta(k)} - c_1 \log n
\]
for some constant $c_1>0$.
\end{theorem}

The proof is dealt with in the following subsection. As already 
mentioned in the introduction, the assumption 
of a fixed starting point of $\cl{n/2}$ one-bits (\ie, $\tr{n/2}$ zero-bits) allows us to concentrate on the essentials; if 
a uniform at random starting point was chosen, then the expected time would at most change by 
a constant~\cite{DoerrDoerrAlgo16}.

\subsection{Bounding the Error}
\label{sec:bounding-error}

This subsection is concerned with the proof of Theorem~\ref{theo:lower}. In particular, most effort is spent on
  establishing the claim
\[ \eta^* - 1 \le c/n\]
for some explicit constant~$c>0$, 
\ie, we bound the additive error of the drift of the potential function $g(k)-g(X_{t+1})$, 
where $X_t=k$ is the current state,
 compared to the lower bound~$1=\Delta(k)/\Delta(k)$ 
established for the drift of the potential function
 at $X_t=k$ in Theorem~\ref{theo:variable-upper}. Here the notions of state (number of 
zero-bits), drift $\Delta(k)$ and transition probabilities are taken over from the preceding section.


Looking back into \eqref{eq:delta-i}, we have already defined the drift (in terms of the number of zero-bits) 
at point~$k$  
and observe that $\Delta(k)$ is monotone increasing in~$k$, which we will use later. Using the notation $p(k,\ell)$ for the transition probability 
from the state of $k$ to~$\ell$ zero-bits, we note that by definition 
\[
\eta(k) = \sum_{\ell=0}^{k-1} p(k,\ell) \sum_{j=\ell+1}^k \frac{1}{\Delta(j)}
\]
and also that 
\[
\sum_{\ell=0}^{k-1} p(k,\ell)  \frac{k-\ell}{\Delta(k)}  = \frac{\Delta(k)}{\Delta(k)} = 1 ,
\]
which is why we pay attention to bounding the terms \begin{equation}
\label{eq:linearize-pot-diff}
p(k,\ell)  \left(\sum_{j=\ell+1}^{k} \frac{1}{\Delta(j)} - \frac{k-\ell}{\Delta(k)}\right)
\end{equation}
with the final aim of showing that
\begin{equation}
\eta(k)-1  =  \sum_{\ell=0}^{k-1} \left(p(k,\ell)  \left(\sum_{j=\ell+1}^{k} \frac{1}{\Delta(j)} - \frac{k-\ell}{\Delta(k)}\right)\right) \le \frac{c}{n}
\label{eq:linearize-pot-full-diff}
\end{equation}
for some sufficiently large constant $c>0$.

We shall define, as in~\cite{HwangEVCO18}, a kind of normalized 
drift that is easier to handle. Here it becomes relevant to 
manipulate the number $n$, so that we write more formally
\begin{align*}
\Delta_n(k) & \coloneqq \Delta(k) \\ 
& = 
\sum_{\ell=1}^k \sum_{j=0}^{\ell} (\ell-j)\binom{k}{\ell}\binom{n-k}{j} \left(\frac{1}{n}\right)^{\ell+j}\left(1-\frac{1}{n}\right)^{n-\ell-j}
\end{align*}

The definition of the normalized drift $\Delta^*$ is then as follows.
\begin{definition}
\label{def:delta-start}
Define, for $k\in\{1,\dots,n+1\}$,
\begin{align*}
    \Delta_n^*(k) 
		& \coloneqq \Delta_{n+1}(k)
    \Bigl(1-\frac1{n+1}\Bigr)^{-n-1} \\
    & = \sum_{\ell=1}^{k} 
    \binom{k}{\ell}\sum_{j=0}^{\ell-1} (\ell-j) 
    \binom{n+1-k}{j} \left(\frac{1}{n}\right)^{j+\ell},
\end{align*}
 Define, for convenience, $\Delta_n^*(0)\coloneqq 0$. 
\end{definition}

From \eqref{eq:linearize-pot-full-diff} we are brought to the task 
of bounding $\frac{1}{\Delta(k)} - \frac{1}{\Delta(k+1)}$, leading 
to Lemma~\ref{lem:diff-1-over-Delta-k} below. To this end, it is crucial 
to bound ${\Delta(k+1)-\Delta(k)}$. While this can 
be achieved in a tedious analysis comparing terms 
in the above-given 
representation of $\Delta(k)$ as a double sum, 
we follow a more elegant approach involving generating functions here.
To this end, 
let $[z^n]f(z)$ denote the coefficient of $z^n$ in the Taylor 
expansion of $f(z)$. 

\begin{lemma} \label{lmm}
For $k\in\{0,\dots,n+1\}$ and $n\ge1$, the relation
\begin{align*}
    \Delta_n^*(k)
    = [z^{-1}]\frac1{(1-z)^2}\Bigl(1+\frac1{nz}\Bigr)^k
	\Bigl(1+\frac zn\Bigr)^{n+1-k}
\end{align*}
holds.
\end{lemma}
\begin{proof}
Rewrite the sum definition of $\Delta_n^*(k)$ as the Cauchy product 
of three series:
\begin{align*}
    \Delta_n^*(k) 
    &= \sum_{\ell = 1}^k 
    \binom{k}{\ell}n^{-\ell}\sum_{j=0}^{\ell-1} (\ell-j) 
    \cdot \binom{n+1-k}{j} n^{-j}\\
    &= \sum_{\ell=0}^{k-1} 
    \binom{k}{\ell}n^{-k+\ell}\sum_{j=0}^{k-\ell-1} (k-\ell-j) 
    \cdot \binom{n+1-k}{j} n^{-j}\\
    &= \sum_{\substack{h,j,\ell\\ h+j+\ell=k-1}}\binom{k}{\ell}
    n^{-k+\ell}\cdot (h+1) 
    \cdot \binom{n+1-k}{j} n^{-j},
\end{align*}
implying that 
\begin{align*}
    \Delta_n^*(k) = [z^{k-1}]
    \Bigl(z+\frac1{n}\Bigr)^k\frac1{(1-z)^2}
    \Bigl(1+\frac zn\Bigr)^{n+1-k}.
\end{align*}
The lemma then follows from the relation 
\begin{align}\label{coeff-shift}
    [z^{k-1}]f(z) = [z^{-1}]z^{-k}f(z)
    \qquad(k\ge1).
\end{align} 
\end{proof}

We shall prove bounds on the difference $\Delta(k+1)-\Delta(k)$ via 
bounding the corresponding difference of the $\Delta^*$-values.

\begin{lemma} 
\label{lem:diff-delta} For $k\in\{0,\dots,n\}$ it holds that 
\begin{align} \label{df-Delta}
    \frac1n\le \Delta_n^*(k+1)-\Delta_n^*(k)
    \le \frac{2e}{n}. 
\end{align}
and for $k\in\{0,\dots,n-1\}$ that 
\[
\frac{1}{en} \le \Delta_n(k+1)  - \Delta_n(k) \le  \frac{2}{n-1}.
\]
\end{lemma}

\begin{proof} 
We prove first~\eqref{df-Delta}. 
From Lemma~\ref{lmm}, we have
\begin{align*}
    \Delta_n^*(k+1)-\Delta_n^*(k)
    &= \frac{1}{n} [z^{-1}]
    \frac1{(1-z)^2}\Bigl(1+\frac1{nz}\Bigr)^{k}
	\Bigl(1+\frac zn\Bigr)^{n-k}
    \Lpa{\frac1z-z}\\
    &= \frac{1}{n}[z^{-1}]
    \frac{1+z}{z(1-z)}\Bigl(1+\frac1{nz}\Bigr)^{k}
	\Bigl(1+\frac zn\Bigr)^{n-k}.
\end{align*}

Thus, by taking the coefficients of the Cauchy product, we obtain
\begin{align}\label{Delta-diff}
    n&\lpa{\Delta_n^*(k+1)-\Delta_n^*(k)} \nonumber\\
    &= \sum_{\ell=0}^k \binom{k}{\ell}n^{-\ell}
    \sum_{j=0}^{n-k}\binom{n-k}{j}n^{-j}
    [z^{\ell-j}]\frac{1+z}{1-z} \\
    &\le 2 \sum_{\ell=0}^k\binom{k}{\ell}n^{-\ell}
    \sum_{j=0}^{\ell-1}\binom{n-k}{j}n^{-j}\nonumber \\
    &< 2\sum_{\ell\ge0}\frac{(k/n)^\ell}{\ell!}
    \sum_{j\ge0}\frac{(1-k/n)^j}{j!}\nonumber \\
    &= 2e^{k/n+1-k/n}=2e.\nonumber
\end{align}
On the other hand, by \eqref{Delta-diff},
\begin{align}\label{Delta-mono}
    n& \lpa{\Delta_n^*(k+1)-\Delta_n^*(k)}\nonumber\\
    & \ge \sum_{\ell=0}^k \binom{k}{\ell}n^{-\ell}
    \sum_{j=0}^\ell \binom{n-k}{j}n^{-j}
    \ge\Lpa{1+\frac1n}^k > 1.
\end{align}
This proves \eqref{df-Delta}.

Recalling the definition \[
\Delta_n(k) = \Delta_{n-1}^*(k) \left(1-\frac{1}{n}\right)^n,
\]
we finally obtain
\[
\Delta_n(k+1)  - \Delta_n(k) \le \left(1-\frac{1}{n}\right)^n \frac{2e}{n-1} \le \frac{2}{n-1}
\]
and 
\[
\Delta_n(k+1)  - \Delta_n(k) \ge \frac{(1-1/n)^n}{n-1}  = \frac{(1-1/n)^{n-1}}{n} \ge \frac{1}{en}
\]
as claimed.
\end{proof}

Recall that we want to investigate the difference
\[
\frac{1}{\Delta(k-1)} - \frac{1}{\Delta(k)} = \frac{\Delta(k) - \Delta(k-1)}{\Delta(k-1)\Delta(k)}
\]
(and later $\frac{1}{\Delta(k-\ell)} - \frac{1}{\Delta(k)} $ for $\ell\ge 1$); 
thus we need bounds on $\Delta(k)$ itself. The following lemma 
gives such bounds along with estimations of the transition probabilities.
We will use 
the notation $p(k,\le j) = \sum_{\ell=0}^{j} p(k,\ell)$ for the probability to change from state~$k$ 
to state at most~$j$.

\begin{lemma}
\label{lem:delta-bounds}
For $k\ge 1$, $e^{-1}\frac{k}{n}\le \Delta(k) \le \frac{k}{n}$. 
Moreover, for $\ell\ge 1$ it holds that 
$p(k,k-\ell)\le p(k,\le k-\ell) \le \binom{k}{\ell}\lpa{\frac{1}{n}}^\ell \le \left(\frac{k}{n}\right)^\ell / \ell!$.
\end{lemma}

\begin{proof}
This proof uses well-known standard arguments. The 
upper bound on the drift follows from considering the expected number of flipping bits among $k$ one-bits 
and the lower bound from looking into steps flipping one bit only. The bound on the transition probability 
considers all mutations flipping at least $\ell$ bits.
\end{proof}

Intuitively, 
the parenthesized term in \eqref{eq:linearize-pot-diff} estimates the error incurred by estimating the potential function 
using the slope at~$k$ for a step of size~$\ell+1$. This error will below in Lemma~\ref{lem:bound-eta-i} be weighted by the probability 
of making a step of such size, more precisely by the probability of jumping from~$k$ to $j=k-\ell-1$. Assembling the previous 
lemmas, we now give a bound 
for the  difference of $1/\Delta(\cdot)$.

\begin{lemma}
\label{lem:diff-1-over-Delta-k}
For $k\ge 1$ and $\ell\in\{1,\dots,k\}$ it holds that
\[
\frac{1}{\Delta(k-\ell)} - \frac{1}{\Delta(k)} \le 
\frac{2e^2 \ell n^2}{k(k-\ell)(n-1)}.
\]
\end{lemma}

\begin{proof}
Using Lemma~\ref{lem:diff-delta} and Lemma~\ref{lem:delta-bounds}, 
\begin{align*}
\frac{1}{\Delta(k-\ell)} - \frac{1}{\Delta(k)} & = 
\frac{\Delta(k)-\Delta(k-\ell)}{\Delta(k)\Delta(k-\ell)} \\
& \le 
\frac{2\ell/(n-1)}{e^{-2} k(k-\ell)/n^2} = \frac{2e^2 \ell n^2}{k(k-\ell)(n-1)} .
\qedhere
\end{align*}
\end{proof}

If we jump from $k\ge 2$ to  $k-\ell-1$ then the parenthesized term 
in~\eqref{eq:linearize-pot-diff} (intuitively incurred by linearizing the potential function 
using the slope at~$k$) equals 
\[
\left(\frac{1}{\Delta(k-\ell)} - \frac{1}{\Delta(k)}\right) 
+ \dots + \left(\frac{1}{\Delta(k-1)} - \frac{1}{\Delta(k)}\right) \le 
\frac{\ell}{\Delta(k-\ell)}-\frac{\ell}{\Delta(k) }.
\]

Finally, we weigh these differences with the respective probabilities and 
 put everything together to bound the whole expression~\eqref{eq:linearize-pot-full-diff}.
\begin{lemma}
\label{lem:bound-eta-i}
For $k\in\{1,\dots,\tr{n/2}\}$ it holds that 
$\eta(k) \le 1+\frac{2e^{5/2}}{n-1}$.
\end{lemma}

\begin{proof}
Using Lemma~\ref{lem:delta-bounds} and Lemma~\ref{lem:diff-1-over-Delta-k}, 
\begin{align*}
    \sum_{\ell=0}^{k-1}\;& p(k, k-\ell-1) 
    \left(\frac{\ell}{\Delta(k-\ell)} 
    - \frac{\ell}{\Delta(k)}\right)\\
    &=\sum_{\ell=1}^k p(k, k-\ell) 
    \left(\frac{\ell-1}{\Delta(k-\ell+1)} 
    - \frac{\ell-1}{\Delta(k)}\right)\\
    &\le \frac{2e^2n^2}{k(n-1)}
    \sum_{\ell=1}^k \binom{k}{\ell}n^{-\ell}
    \frac{(\ell-1)^2}{k-\ell+1}\\
    &\le \frac{2e^2n^2}{k(n-1)}
    \sum_{\ell=0}^k\binom{k}{\ell}n^{-\ell}
    \frac{\ell(\ell-1)}{k-\ell+1}.
\end{align*}
Applying the integral representation $\frac1a=\int_0^1 t^{a-1}\dd t$ for $a>0$, we obtain 
\begin{align*}
    \sum_{\ell=0}^k\binom{k}{\ell}n^{-\ell}
    \frac{\ell(\ell-1)}{k-\ell+1}
    &= \int_0^1 \sum_{\ell=0}^k \binom{k}{\ell}\ell(\ell-1)n^{-\ell}
    t^{k-\ell}\dd t\\
    &= \frac{k(k-1)}{n^2}\int_0^1\Lpa{t+\frac1n}^{k-2}\dd t\\
    &\le \frac{k}{n^2}\Lpa{1+\frac1n}^{k-1}.
\end{align*}
Thus 
\begin{align*}
    &\sum_{\ell=0}^{k-1} p(k, k-\ell-1) 
    \left(\frac{\ell}{\Delta(k-\ell)} 
    - \frac{\ell}{\Delta(k)}\right)\\
    &\qquad\le \frac{2e^2n^2}{k(n-1)}
    \cdot \frac{k}{n^2}\Lpa{1+\frac1n}^{k-1}\\
    &\qquad= \frac{2e^2}{n-1}\Lpa{1+\frac1n}^{k}
    \le \frac{2e^{5/2}}{n-1}.
\end{align*}
\end{proof}

We are now ready to complete the proof of Theorem~\ref{theo:lower}.

\begin{proofof}{Theorem~\ref{theo:lower}}
According to Theorem~\ref{theo:var-lower-error}, we 
have
\begin{align*}
\expect{T\mid X_0=n/2} & \ge \sum_{k=1}^{\tr{n/2}}  \frac{1}{ \eta^*\Delta(k)} 
 \ge \sum_{k=1}^{\tr{n/2}}\frac{1}{\Delta(k)} \left(1-\frac{\eta^*-1}{\eta^*}\right),
\end{align*}
and by Lemma~\ref{lem:bound-eta-i}
\[
\eta^* \le 1+\frac{2e^{5/2}}{n-1}.
\]
Since $\sum_{k=1}^{\tr{n/2}} 1/\Delta(k) \le enH_{\tr{n/2}} \le en\log n$ as observed in 
Section~\ref{sec:triv-upper}, we altogether obtain
\begin{align*}
\expect{T\mid X_0=n/2} & \ge \sum_{k=1}^{\tr{n/2}}  \frac{1}{ \Delta(k)} - \frac{(en\log n)(2e^{5/2})}{n-1} \\ & \ge  
\sum_{k=1}^{\tr{n/2}}  \frac{1}{ \Delta(k)} - (4e^{7/2})\log n,
\end{align*}
where the last inequality used $n\ge 2$. 
Altogether, the theorem has been established with $c_1=4e^{7/2}\approx 132.56$.
\end{proofof}

In conjunction with Section~\ref{sec:triv-upper}, we 
have determined the expected runtime of the \ea on \om (starting in state $\tr{n/2}$, \ie, with $\cl{n/2}$ one-bits) up to 
an additive term bounded by $c_1\log n$.
As already mentioned in the introduction, terms of even lower order 
down to $O((\log n)/n)$ have been determined in~\cite{HwangEVCO18} by a more technical analysis. Our result 
features a non-asymptotic error bound.

\section{Improving the $\bf \sum_{k=1}^{\tr{n/2}}\frac{1}{\Delta(k)}$ Bound}
\label{sec:upper}

The upper bound $\sum_{k=1}^{\tr{n/2}} 1/\Delta(k)$ derived Section~\ref{sec:triv-upper} precisely 
characterizes the expected runtime of the \ea on \om, but is a slight 
overestimation resulting from the inequality $1/h(X_{t+1})\ge 1/h(X_t)$ in the proof 
of Theorem~\ref{theo:variable-upper}; intuitively this corresponds to estimating the progress 
from state $X_t$ via a linearized potential function of slope $1/h(X_t)$, which is the derivative 
of $g$  at $X_t$. 

We can improve the bound on the expected runtime  
by estimating the error stemming from this inequality and will gain  a logarithmic term. To this end,
we study the following simple analogue of Theorem~\ref{theo:var-lower-error}.

\begin{theorem}
\label{theo:var-upper-error}[Variable drift, upper bound, with error bound]
Let $(X_t)_{t\ge 0}$ be a non-increasing Markov process on the state space $\{0,\dots,N\}$ for some integer~$N$.
Let  $\Delta\colon \{1,\dots,N\} \to \R^+$ be a function satisfying  
$E(X_t-X_{t+1} \mid X_t=k) \ge \Delta(k)$ for $k\in\{1,\dots,N\}$. Let 
\[
\eta(k) \coloneqq  \expect{\sum_{j=X_{t+1}+1}^{k} \frac{1}{\Delta(j)}  \bigm| X_t=k}  
\]
and
\[
\eta^* \coloneqq \min_{k=1,\dots,N} \eta(k).
\]
 Then it holds for the first hitting time
$T:=\min\{t\mid X_t=0\}$ that
\[
\expect{T\mid X_0} \le
\sum_{k=1}^{X_0}  \frac{1}{ \eta^*\Delta(k)} .
\]
\end{theorem}

\begin{proof}
We proceed analogously to the proof of Theorem~\ref{theo:var-upper-error},  
use the potential function $g(k)=\sum_{j=1}^k 1/\Delta(j)$ and apply additive drift 
analysis (Theorem~\ref{theo:additive-drift}, part~1) 
with the lower bound 
$\eta^*$ on its drift.
\end{proof}

We state our improved result, carrying over notation from previous sections such as 
the definition of the drift $\Delta(k)$ with respect to \ea and \om.

\begin{theorem}[Improved Upper Bound]
\label{theo:improved-upper}
Let $n\ge 4$. 
Then the expected optimization time of the \ea on \om (starting at $\cl{n/2}$ ones) is at most
$\sum_{k=1}^{\tr{n/2}} 1/\Delta(k) - c_2\log n$ for some constant~$c_2>0$.
\end{theorem}

To prove this result, we need to invert a statement from 
Section~\ref{sec:bounding-error}.
\begin{lemma}
\label{lem:diff-1-over-delta-lower}
For $ k\in\{2,\dots,n/2\}$, \[
\frac{1}{\Delta(k-1)} - \frac{1}{\Delta(k)} \ge \frac{n}{ek^2}.
\]
\end{lemma}

\begin{proof}
We proceed similarly to the proof of 
Lemma~\ref{lem:diff-1-over-Delta-k} but 
aim at lower bounds. First, we recall from Lemma~\ref{df-Delta} that 
\[
\Delta(k)-\Delta(k-1) \ge  \frac{1}{en}.
\]
Now, using the upper bound 
$\Delta(k)\le k/n$ from Lemma~\ref{lem:delta-bounds}, we obtain
\[
\frac{1}{\Delta(k-1)} - \frac{1}{\Delta(k)} = \frac{\Delta(k)-\Delta(k-1)}{\Delta(k)\Delta(k-1)} 
\ge \frac{1}{en (k/n)^2} = \frac{n}{ek^2},
\]
which concludes the proof.
\end{proof}
We can now present the proof of the improved upper bound. 

\begin{proofof}{Theorem~\ref{theo:improved-upper}}
The aim is to apply Theorem~\ref{theo:var-upper-error} for 
some $\eta^*=1+c/n$, where $c>0$ is constant. 
Since state~$1$ is special in that it only has one possible successor, we 
consider 
$T_1:=\min\{t\mid X_t\le 1\}$ instead and the following straightforward generalization of 
the theorem:
\[
\expect{T_1\mid X_0} \le
\sum_{k=2}^{X_0}  \frac{1}{ \eta^*\Delta(k)}, 
\]
where $\eta^* \coloneqq \min_{k=2,\dots,n} \eta(k)$. 
This implies  
\[
\expect{T_0\mid X_0=n/2} \le
\frac{1}{\Delta(1)} + \sum_{k=2}^{\tr{n/2}}  \frac{1}{ \eta^*\Delta(k)}
\]
since the expected transition time from state~$1$ to~$0$ is exactly $1/\Delta(1)$. 

We now show that $\eta(k)\ge 1+c_1/n$ for some constant $c_1>0$ and $k\in\{2,\dots,\tr{n/2}\}$. 
Note that (conditioning on $X_t=k$ everywhere) 
\begin{align*}
\eta(k) & = \expect{\sum_{j=X_{t+1}+1}^{k} \frac{1}{\Delta(j)} } \\
& = \expect{\sum_{j=X_{t+1}+1}^{k} \frac{1}{\Delta(j)} \mid  X_{t+1}<k-1 } \,\, \prob{X_{t+1}<k-1} \\ 
& \qquad + 
 \frac{\expect{(k-X_{t+1})\mid X_{t+1}\ge k-1}}{\Delta(k)} \,\prob{X_{t+1}\ge k-1} 
\\
& = 
\expect{\sum_{j=X_{t+1}+1}^{k} \frac{1}{\Delta(j)} \mid  X_{t+1}<k-1 } \, \prob{X_{t+1}<k-1} \\
& \qquad + \frac{\expect{(k-X_{t+1})\indic{X_{t+1}\ge k-1}}}{\Delta(k)}.
\end{align*}

The first term on the right-hand side can be bounded from below by
\[
 \left(\frac{\expect{(k-1-X_{t+1}) \mid X_{t+1}<k-1}}{\Delta(k-1)} + \frac{1}{\Delta(k)}\right)\prob{X_{t+1}<k-1} 
\]
since $\Delta(k)$ is non-decreasing. Using Lemma~\ref{lem:diff-1-over-delta-lower}, 
the last expression is further bounded from below by 
\begin{align*}
& \left(\expect{(k-1-X_{t+1}) \mid X_{t+1}<k-1}  \left(\frac{1}{\Delta(k) }   
+ \frac{n}{ek^2}\right)+\frac{1}{\Delta(k)}\right) \\
& \hspace*{6cm} \mbox{}\cdot \prob{X_{t+1}<k-1},
\end{align*}
which, using 
\[\prob{X_{t+1}<k-1} \ge \frac{e^{-1}k(k-1)}{2n^2} \ge e^{-1} \frac{k^2}{4n^2}\]
and \[
\expect{(k-1-X_{t+1}) \mid X_{t+1}<k-1 }\ge 1,
\] 
is at least
\[
 \left(\frac{\expect{(k-X_{t+1}) \indic {X_{t+1}<k-1} }}{\Delta(k) }\right) + \frac{e^{-2}}{4n}.
\]
Putting everything together, we 
have
\begin{align*}
\eta(k)   & \ge 
\frac{\expect{(k-X_{t+1})\indic{X_{t+1}\ge k-1}}}{\Delta(k)} \\
& \quad +  
 \left(\frac{\expect{(k-X_{t+1}) \indic{X_{t+1}<k-1} }}{\Delta(k) }\right)  + \frac{e^{-2}}{4n}  \\
& = \frac{\Delta(k)}{\Delta(k)} + \frac{e^{-2}}{4n} = 1+\frac{e^{-2}}{4n},
\end{align*}
so  $\eta^*\ge 1+e^{-2}/(4n)$. 
We conclude the proof by noting that 
\[
\sum_{k=2}^{\tr{n/2}} \frac{1}{\Delta(k)(1+e^{-2}/(4n))} \le \sum_{k=2}^{\tr{n/2}} \frac{1}{\Delta(k)} -  
\sum_{k=2}^{\tr{n/2}} \frac{e^{-2}}{4n (1+e^{-2}/4) \Delta(k)},
\]
which, using 
$\sum_{k=2}^{\tr{n/2}} \frac{1}{\Delta(k)} \ge n (H_{\tr{n/2}}-1) \ge n(\log n)/3$ for $n\ge 4$, 
amounts to 
\[
\expect{T\mid X_0=n/2} \le \sum_{k=1}^{\tr{n/2}} \frac{1}{\Delta(k)} - \frac{e^{-2}}{12 (1+e^{-2}/4)}
\log n.
\]
Hence, we can set $c_2=\frac{e^{-2}}{12 (1+e^{-2}/4)}\approx 1/91.69$.
\end{proofof}

\section{Formulas for The Exact Optimization Time}
\label{sec:exact}

In light of the Theorems~\ref{theo:var-lower-error} and~\ref{theo:var-upper-error} 
one might wonder whether one should try to choose a potential function that makes the ``error'' $\eta^*$ vanish 
and leads to 
a drift of exactly~$1$. It is well known \cite{Lehre12DriftTutorial,LenglerDriftSurveyArxiv} that 
letting $g(k)$ be the expected remaining optimization time from state~$k$ actually 
achieves this. 

In this section, we briefly investigate how to choose~$g(k)$ with respect to our  
setting of \ea and \om. We will obtain formulas that can also be derived manually, so the 
result is by no means new. However, it is still interesting to see that it can be derived via drift analysis.
This will turn out in the proof of the following theorem.

\begin{theorem}
\label{theo:times-transition}
Let $(X_t)_{t\ge 0}$ be a non-increasing Markov process on the state space $\{0,\dots,N\}$ for some integer~$N$ 
and denote by $p(k,j)$ the transition probability from state~$k$ to state~$j$.
 Let the function $g(k)$ be recursively defined by $g(0)\coloneqq 0$ and 
for $k\ge 1$:
\[
g(k)\coloneqq \frac{1+\sum_{j=1}^{k-1}p(k,j) g(j)}{\sum_{j=0}^{k-1}p(k,j)}
\]
 Then it holds for the first hitting time
$T:=\min\{t\mid X_t=0\}$ that
\[
\expect{T\mid X_0} = 
  g(X_0).
\]
\end{theorem}

\begin{proof}
We shall use additive drift analysis (Theorem~\ref{theo:additive-drift}), which 
gives the exact expected hitting time if $\expect{g(X_t)-g(X_{t+1})\mid \filt} = \delta$, \ie, if
both the first and the second cases of the theorem hold.
 
We compute 
\begin{align*}
&  \expect{g(k)-g(X_{t+1})\mid X_t=k}  = \sum_{j=0}^{k-1} p(k,j) (g(k)-g(j)) \\
& \qquad\qquad = (1-p(k,k)) g(k) - \sum_{j=0}^{k-1} p(k,j)g(j) \\
& \qquad\qquad
= (1-p(k,k)) \frac{1+\sum_{j=0}^{k-1}p(k,j) g(j)}{\sum_{j=0}^{k-1}p(k,j)}
- \sum_{j=0}^{k-1} p(k,j)g(j)\\
& \qquad\qquad= \left(1+\sum_{j=0}^{k-1}p(k,j) g(j)\right) - 
\sum_{j=0}^{k-1} p(k,j)g(j) = 1,
\end{align*}
with the the definition of $g(k)$ plugged in the third equality. 
 Hence, by Theorem~\ref{theo:additive-drift} the expected hitting time of state~$0$ from state~$X_0$ equals $g(X_0)/1$.
\end{proof}

That $g(k)$ equals the expected first hitting time from state~$k$ to state~$0$ can also be proved 
in an elementary induction. By writing
\[
g(k) = \frac{1}{\sum_{j=0}^{k-1}p(k,j)} + \sum_{j=1}^{k-1}\frac{p(k,j)}{\sum_{j=0}^{k-1}p(k,j)}\, g(j)
\]
we realize that the first term is the expected time to leave state~$k$ and the second term is a weighted sum 
of the remaining optimization times from smaller state, weighted by the respective transition probabilities 
conditional on leaving state~$k$. Such formulas can also be derived by inverting matrices obtained from the transition 
probabilities of the underlying Markov chain \cite{ChicanoEC15}.

We note that estimations of hitting times in finite search spaces based 
on the transition probabilities were recently presented in K{\"o}tzing and Krejca~\cite{KoetzingKrejcaPPSN2018-finite}. These estimations 
are not recursively defined and easy to evaluate. However, as the underlying scenario does 
not allow big jumps towards the optimum 
when estimating the hitting time from below, 
tight formulas for the \ea on \om cannot be proved with this approach.

We exemplarily apply Theorem~\ref{theo:times-transition} to our scenario of the \ea on \om. 
Using the transition probabilities  
\[
p(k,j)= \sum_{\ell=0}^{\min\{j,n-k\}} \binom{k}{k-j+\ell}\binom{n-k}{\ell} \left(\frac{1}{n}\right)^{k-j+2\ell}\left(1-\frac{1}{n}\right)^{n-(k-j)-2\ell} 
\]
we obtain $g(0)=0$, $g(1) = n(1-1/n)^{1-n}$, and 
\begin{footnotesize}
\begin{align*}
g(2)&=\frac{(3n^3-8n^2+6n-1)(1-1/n)^{1-n}}{2n^2-2n-1}\\[2ex]
g(3)&=\frac{
(22n^7-114n^6+203n^5-117n^4-38n^3+49n^2-7n+2)(1-1/n)^{1-n}
 }
 {12n^6-36n^5+4n^4+60n^3-23n^2-21n-2}
 \\
 &\vdots
\end{align*}
\end{footnotesize}

While these expansions obviously reflect the well-known estimate $g(k)=(1\pm o(1))en H_k$, they do not seem
readily useful in expressing the expected runtime of the \ea on \om in a closed-form 
formula depending on~$n$.

\section{The Asymptotics of the Partial Sum $\bf \sum_{k=1}^{\tr{n/2}}\frac{1}{\Delta(k)}$}
\label{sec:asymptotic-expansion}

The purpose of this section is to analyze  more precisely
how far the sum of inverse drifts $\sum_{k=1}^{\tr{n/2}}1/\Delta(k)$
differs from the expected optimization time
\[
    \expect{T\mid X_0=\tr{n/2}}
    = en \log n - C_1 n + (e/2)\log n + O(1)
\]
derived in~\cite{HwangEVCO18}. We know from the preceding analysis
that the sum of inverse drifts overestimates $\expect{T\mid X_0=\tr{n/2}}$
by a $\Theta(\log n)$-term. We will prove the following asymptotic
approximation for the sum of inverse drifts, which, when
compared with~\eqref{eq:exact-hwang}, shows their logarithmic
difference. 

\begin{theorem} \label{thm:S1D} For large $n$,
\begin{align}\label{S1D}
    \sum_{k=1}^{\tr{n/2}}\frac1{\Delta_n(k)}
    = en\log n -C_1n + e\log n +O(1),
\end{align}    
where
\begin{align}\label{C1}
    C_1 := -e\left(\gamma-\log 2+\int_0^{1/2}
    \left(\frac1{S_1(t)}-\frac1t\right)\dd t\right)
    \approx 1.89254\dots
\end{align}    
is the same linear constant appearing in \eqref{eq:exact-hwang}. Here
\begin{align}\label{Srz}
    S_r(z) := \sum_{\ell\ge0}\frac{z^\ell}{\ell!}
    \sum_{j=0}^{\ell-1}(\ell-j)^r \frac{(1-z)^j}{j!}
    \qquad(r\ge0; z\in\mathbb{C}).
\end{align}
\end{theorem}
Note that if we multiply the left-hand side of \eqref{S1D} by 
$e^{-1/(2n)}$, then the difference with \eqref{eq:exact-hwang} is 
bounded, namely, 
\[
    e^{-1/(2n)}\sum_{k=1}^{\tr{n/2}}\frac1{\Delta_n(k)}
    = en\log n -C_1\,n + \frac{e}2\log n +O(1).
\]

To prove Theorem~\ref{thm:S1D}, we use the techniques of generating functions
and Euler-Maclaurin summation formula, which are conceptually and
methodologically simpler than the asymptotic resolution of the
recurrences used in~\cite{HwangEVCO18}. The following lemma can be
obtained in style similar to Lemma~\ref{lem:delta-bounds}. Since it
is with respect to the normalized $\Delta^*$, we give a
self-contained proof.

\begin{lemma} For $k\in\{0,\dots,n+1\}$, 
\begin{align}\label{ineq-Delta}
    \Lpa{1+\frac1n}^{k-1} \frac kn
    \le \Delta_n^*(k)\le \Lpa{1+\frac1n}^{n} \frac kn.
\end{align}    
\end{lemma}
\begin{proof} 
By definition
\begin{align*}
    \Delta_n^*(k)
    &= \frac kn\sum_{\ell=0}^{k-1} 
    \binom{k-1}{\ell}n^{-\ell}\sum_{j=0}^{\ell}
    \frac{\ell+1-j}{\ell+1}
    \binom{n+1-k}{j} n^{-j}\\
    &\le \frac kn\sum_{\ell=0}^{k-1} 
    \binom{k-1}{\ell}n^{-\ell}\sum_{j=0}^{\ell}
    \binom{n+1-k}{j} n^{-j}\\
    &\le \frac kn \sum_{\ell=0}^{k-1} 
    \binom{k-1}{\ell}n^{-\ell}
     \sum_{j=0}^{n+1-k} 
    \binom{n+1-k}{j} n^{-j}\\
    &= \Lpa{1+\frac1n}^{n} \frac kn < e\,\frac kn.
\end{align*}
On the other hand, 
\begin{align*}
    \Delta_n^*(k)
    &\ge \frac kn\sum_{\ell=0}^{k-1} 
    \binom{k-1}{\ell}n^{-\ell}
    = \Lpa{1+\frac1n}^{k-1} \frac kn.
\end{align*}  
\end{proof}
Note that \eqref{ineq-Delta} becomes an identity when $k=0$ and 
$k=n+1$.

The crucial lemma we need to prove \eqref{S1D} is given as follows. 
\begin{lemma} Let $\ve>0$. Then for $1\le k\le (1-\ve)n$, 
\begin{align}\label{ST}
    \Delta_n^*(k)
    = S_1(\alpha) + \frac{T_1(\alpha)}{n} 
    + O\lpa{n^{-2}},
\end{align}
where $\alpha=k/n$ and 
\begin{equation}\label{T1}
    \begin{split}
    T_1(\alpha) &= \tfrac12S_1(\alpha)-2\alpha S_0(\alpha)
	-\alpha \,I_0\lpa{2\sqrt{\alpha(1-\alpha)}}\\
	&\qquad -\sqrt{\alpha(1-\alpha)}\,
    I_1\lpa{2\sqrt{\alpha(1-\alpha)}}.
    \end{split}
\end{equation}
Here the $I_j$'s represent the modified Bessel functions. 
\end{lemma}
It is possible to extend further the range in $k$, but we do not need 
it here. 
\begin{proof} 
First for small $k$, we have, by Definition~\ref{def:delta-start} and 
direct expansion, 
\begin{align}\label{Dnk-small}
    \Delta_n^*(k)
	= \frac{k}{n}+\frac{3k(k-1)}{2n^2}
    +O\lpa{k^3n^{-3}},
\end{align}
which holds uniformly for $1\le k=o(n)$. A simple, readily codable 
procedure to derive this is as follows. Assuming $k$ to be fixed and 
expanding 
\[
    \Bigl(1+\frac1{nz}\Bigr)^k
    \Bigl(1+\frac zn\Bigr)^{n+1-k}
    = e^z + \frac{e^z}{2n} \left(\frac{2k}z-2(k-1)z-z^2\right)
    +\cdots,
\]
for large $n$. Then multiplying both sides by $(1-z)^{-2}$ and
computing the coefficient of $z^{-1}$ term by term (corresponding to
the residue of the integrand in the Cauchy integral), giving
\begin{align*}
    [z^{-1}]\frac{e^z}{(1-z)^2} &= 0,\\
    [z^{-1}]\frac{e^z}{2n(1-z)^2} 
    \left(\frac{2k}z-2(k-1)z-z^2\right)
    &= \frac kn,\\
    [z^{-1}]\frac{e^z}{2n(1-z)^2} 
    \left(\frac{k(k-1)}{z^2}-2k(k-1)z\right)
    &= \frac {3k(k-1)}{2n^2},\\
    &\cdots
\end{align*}
On the other hand, by the Taylor expansions
\begin{equation}\label{ST-z}
    S_1(z) = z +\tfrac32z^2+\tfrac5{12}z^3
    +\cdots
    \text{ and }
    T_1(z) = -\tfrac32z-\tfrac74z^2-\tfrac18z^3
    +\cdots,
\end{equation}
we see that 
\[
    S_1(\alpha) + \frac{T_1(\alpha)}{n} 
    = \alpha +\frac32\,\alpha^2-\frac{3\alpha}{2n}
    +O\lpa{\alpha^3+\alpha^2\,n^{-1}},
\]
consistent with \eqref{Dnk-small}. This proves \eqref{ST} when 
$k=o(n)$. 

Now consider larger values of $k$ and write $k=\alpha n$, where
$\alpha\in[\ve,1-\ve]$. Then
\begin{align*}
    &\alpha n \log\Lpa{1+\frac1{nz}}
    +(1-\alpha) n\log\Lpa{1+\frac zn} \\
    &\qquad\eqqcolon \frac\alpha z+(1-\alpha)z
    -\frac{\alpha+(1-\alpha)z^4}{2nz^2}
    +E_0(z),
\end{align*}
where
\begin{align*}
    E_0(z) &= \sum_{\ell\ge2}\frac{(-1)^{\ell}}{n^\ell}
    \left(\alpha\,\frac{z^{-\ell-1}}{\ell+1}
    +(1-\alpha)\frac{z^{\ell+1}}{\ell+1} \right) \\
    &= O\left( \frac{\alpha |z|^{-3}+ 
    (1-\alpha)|z|^3}{n^2}\right).
\end{align*}
By the inequality
\[
    |e^z-1| = \left|z\int_0^1 e^{tz}\dd t\right|
    \le |z|e^{|z|}\qquad(z\in\mathbb{C}),
\]
we have 
\begin{align*}
	&\left|\Bigl(1+\frac1{nz}\Bigr)^k
	\Bigl(1+\frac zn\Bigr)^{n-k}
    -e^{\frac\alpha z+(1-\alpha)z
    -\frac{\alpha z^{-2}+(1-\alpha)z^2}{2n}}\right|\\
    &\qquad \le |E_0(z)| e^{|E_0(z)|}
    \Bigl|e^{\frac\alpha z+(1-\alpha)z
    -\frac{\alpha z^{-2}+(1-\alpha)z^2}{2n}}\Bigr|.
\end{align*}
The error is then estimated by using the Cauchy integral 
representation 
\begin{align*}
    &[z^{-1}]\frac{1}{(1-z)^2}
    \Lpa{1+\frac1{nz}}^k\Lpa{1+\frac zn}^{n+1-k}\\
    &\qquad= \frac{1}{2\pi i}\oint_{|z|=r}
    \frac1{(1-z)^2}\Bigl(1+\frac1{nz}\Bigr)^k
	\Bigl(1+\frac zn\Bigr)^{n+1-k}\dd z,
\end{align*}
so that ($0<r<1$)
\begin{align*}
    &\left|\oint_{|z|=r}
    \frac{|E_0(z)| e^{|E_0(z)|}}{|1-z|^2}
    \Bigl|e^{\frac\alpha z+(1-\alpha)z
    -\frac{\alpha z^{-2}+(1-\alpha)z^2}{2n}}
    \Bigl(1+\frac zn\Bigr)\Bigr| \dd z\right|\\
    &=O\left(n^{-2}\int_{-\pi}^{\pi}
    \frac{\alpha r^{-2} + (1-\alpha)r^4}{(1-r)^2}\,
    e^{\frac \alpha r\cos t+(1-\alpha)r\cos t}\right)\dd t\\
    &=O\lpa{n^{-2}},
\end{align*}
since $r$ is away from $1$. Thus 
\begin{align*}
    \Delta_n^*(k)
    &= [z^{-1}]\frac{1+\frac zn}{(1-z)^2}
    \,e^{\frac\alpha z+(1-\alpha)z
    -\frac{\alpha z^{-2}+(1-\alpha)z^2}{2n}}
    +O\lpa{n^{-2}}.
\end{align*}
By the same argument, we have
\begin{align*}
    \Delta_n^*(k)
    &= [z^{-1}]\frac{e^{\frac\alpha z+(1-\alpha)z}}{(1-z)^2}
	\left(1-\frac{\alpha-2z^3+(1-\alpha)z^4}{2nz^2}
	\right)+O\lpa{n^{-2}}.
\end{align*}
The lemma will then follow from the relations
\begin{align}\label{S1a}
    S_1(\alpha) = [z^{-1}]
    \frac{e^{\frac\alpha z+(1-\alpha)z}}{(1-z)^2},
\end{align}
and 
\begin{align}\label{T1a}
	T_1(\alpha) &= [z^{-1}]
	\frac{e^{\frac\alpha z+(1-\alpha)z}}
	{(1-z)^2}\cdot\frac{-\alpha+2z^3-(1-\alpha)z^4}{2z^2}.
\end{align}
To prove \eqref{S1a}, we expand the factor $e^{\frac\alpha z}$ and
take the coefficient term by term, yielding
\begin{align*}
    [z^{-1}]\frac{e^{\frac\alpha z+(1-\alpha)z}}{(1-z)^2}
    &= \sum_{\ell\ge0}\frac{\alpha^\ell}{\ell!}
    [z^{\ell-1}]\frac{e^{(1-\alpha)z}}{(1-z)^2}\\
    &= \sum_{\ell\ge0}\frac{\alpha^\ell}{\ell!}
    \sum_{j=0}^{\ell-1}(\ell-j)\frac{(1-\alpha)^j}{j!}
    = S_1(\alpha).
\end{align*}
Similarly,
\begin{align}\label{S0a}
    S_0(\alpha) = [z^{-1}]
    \frac{e^{\frac\alpha z+(1-\alpha)z}}{1-z},
\end{align}
and by the decomposition,
\[
    \frac{-\alpha+2z^3-(1-\alpha)z^4}{z^2(1-z)^2}
	= \frac1{(1-z)^2}-\frac{4\alpha}{1-z}
	-(1-\alpha)-\frac{2\alpha}z-\frac{\alpha}{z^2},
\]
we obtain 
\begin{align*}
    &[z^{-1}]
    \frac{e^{\frac\alpha z+(1-\alpha)z}}
    {(1-z)^2}\cdot\frac{-\alpha+2z^3-(1-\alpha)z^4}{2z^2} \\
    &\quad= \frac12
    \sum_{\ell\ge0}\frac{\alpha^\ell}{\ell!}
	\sum_{j=0}^{\ell-1}\frac{(1-\alpha)^j}{j!}
	(\ell-j -4\alpha)\\
	&\qquad-(1-\alpha)\sum_{\ell\ge1}\frac{\alpha^\ell}{\ell!}
	\cdot\frac{(1-\alpha)^{\ell-1}}{(\ell-1)!}
	-2\alpha \sum_{\ell\ge0}\frac{\alpha^\ell}{\ell!}
	\cdot\frac{(1-\alpha)^{\ell}}{\ell!}\\
	&\qquad-\alpha \sum_{\ell\ge0}\frac{\alpha^\ell}{\ell!}
	\cdot\frac{(1-\alpha)^{\ell+1}}{(\ell+1)!},
\end{align*}
which equals $T_1(\alpha)$ by properly grouping the terms. 
This proves the lemma. 
\end{proof}

As we will see below, finer calculations give
\begin{align}\label{ST2}
    \Delta_n^*(k)
    = S_1(\alpha) + \frac{T_1(\alpha)}{n} 
    + \frac{T_2(\alpha)}{n^2} 
    + O\lpa{n^{-3}},
\end{align}
where
\begin{equation}\label{T2a}
    \begin{split}
        T_2(\alpha)
        &:= -\frac{S_1(\alpha)}{24}+\alpha S_0(\alpha)
    	+\frac{1+6\alpha}{12} \,I_0\lpa{2\sqrt{\alpha(1-\alpha)}}\\
    	&\qquad-\frac{1-10\alpha+4\alpha^2}{12\sqrt{\alpha(1-\alpha)}}
        \,I_1\lpa{2\sqrt{\alpha(1-\alpha)}}.
    \end{split}
\end{equation}
In particular, when $\alpha\to0$, we have $T_2(\alpha)
=\frac43\alpha+\frac{215}{144}\alpha^2+\frac{13}{192}\alpha^3+\cdots$.

To obtain formula \eqref{T2a} for $T_2(\alpha)$ we begin with the 
expression
\[
	T_2(\alpha) = [z^{-1}]
	e^{\frac\alpha z+(1-\alpha)z}\cdot 
	\frac{W_\alpha(z)}{(1-z)^2},
\]
where
\begin{footnotesize}
\[
    W_\alpha(z) 
    := \frac{3\alpha^2+8\alpha z-12\alpha z^3
    +6\alpha(1-\alpha)z^4-4(1-\alpha)z^7+3(1-2\alpha)z^8}
    {24 z^4}.
\]
\end{footnotesize}
By the decomposition 
\begin{footnotesize}\def\arraystretch{1}
    \begin{align*}
        \frac{W_\alpha(z)}{(1-z)^2}
        &= \left(\begin{array}{l}
            -\frac1{24}{(1-z)^2}\\
            +\frac{\alpha}{1-z}
        \end{array}\right)
        +\left(
        \begin{array}{l}
            \frac{\alpha^2}8z^{-4} \\
            +\frac{\alpha(4+3\alpha)}{12}z^{-3}\\
            +\frac{\alpha(16+9\alpha)}{24}z^{-2}\\
            +\frac{\alpha(1-\alpha)}2z^{-1}\\
        \end{array}\right)
        +\left(\begin{array}{l}
            \frac{(1-\alpha)(1-9\alpha)}{24}\\
            +\frac{(1-\alpha)(1-3\alpha)}{12}z\\
            +\frac{(1-\alpha)^2}8z^2
        \end{array}\right) ,   
    \end{align*}
\end{footnotesize}
we then derive \eqref{T2a} by a term-by-term translation using the 
relations \eqref{S1a}, \eqref{S0a} and 
\begin{align*}
    [z^{-1}] z^m e^{\frac\alpha z+(1-\alpha)z}
    &= \sum_{\ell\ge \max\{0,-m+1\}}
    \frac{\alpha^\ell(1-\alpha)^{m+\ell-1}}
    {\ell!(m+\ell-1)!} \\
    &= \left(\frac{1-\alpha}\alpha\right)^{(m-1)/2}
    I_{m-1}\lpa{2\sqrt{\alpha(1-\alpha)}},
\end{align*}
for $m\in\mathbb{Z}$.

\paragraph{Proof of Theorem~\ref{thm:S1D}.}
Substituting the expansion \eqref{ST2} into the partial sum 
\[
    Q_{\tr{n/2}} = \sum_{k=1}^{\tr{n/2}}
	\frac1{\Delta_n^*(k)},
\]
and using the expansion
\begin{align*}
    &\frac1{S_1(\alpha)+\frac{T_1(\alpha)}{n} 
    + \frac{T_2(\alpha)}{n^2} 
    + O\lpa{n^{-3}}} \\
    &\qquad= \frac1{S_1(\alpha)}
    -\frac{T_1(\alpha)}{nS_1(\alpha)^2}
    -\frac{S_1(\alpha)T_2(\alpha)-T_1(\alpha)^2}
    {n^2S_1(\alpha)^3}+O\lpa{\alpha^{-2}n^{-3}}
\end{align*}
we obtain 
\begin{align*}
    Q_{\tr{n/2}} &= \sum_{k=1}^{\tr{n/2}}
	\frac1{S_1\lpa{\frac kn}}
    -\frac1{n}\sum_{k=1}^{\tr{n/2}}
	\frac{T_1\lpa{\frac kn}}{S_1\lpa{\frac kn}^2}
    + E_1(n),
\end{align*}
where 
\begin{align*}
    E_1(n) &= -\frac1{n^2} \sum_{k=1}^{\tr{n/2}}
    \frac{S_1\lpa{\frac kn}T_2\lpa{\frac kn}
    -T_1\lpa{\frac kn}^2}{S_1\lpa{\frac kn}^3}\\
    &\qquad\qquad\qquad    +O\left(n^{-3}\sum_{k=1}^{\tr{n/2}}\frac{n^2}{k^2}\right).
\end{align*}
By the local expansion 
\[
    \frac{S_1(\alpha)T_2(\alpha)-T_1(\alpha)^2}
    {S_1(\alpha)^3} = -\frac{11}{12}\, \alpha^{-1}
    +\frac{341}{144}+\cdots,
\]
we deduce that 
\[
    E_1(n) = O\left(n^{-1}\sum_{k=1}^{\tr{n/2}}k^{-1}+
    n^{-1}\right)
    = O\lpa{n^{-1}\log n}. 
\]
On the other hand, since most contribution to the sums come from 
terms with small $k$, we deduce, by using the expansion 
\[
    \frac{T_1(\alpha)}{S_1(\alpha)^2}
    = -\frac{3}{2\alpha} +\frac{11}4 -\frac{15}4\,\alpha
    +\cdots,
\]
and the boundedness of $\frac{\alpha T_1(\alpha)}{S_1(\alpha)^2}$ on 
the unit interval, that 
\[
    -\frac1{n}\sum_{k=1}^{\tr{n/2}}
	\frac{T_1\lpa{\frac kn}}{S_1\lpa{\frac kn}^2}
    = \frac32\,H_{\tr{n/2}} +O(1)
    = \frac32\,\log n +O(1).
\]
Define
\[
    R(z) := \frac1{S_1(z)}-\frac1z,
\]
which is bounded in the unit interval. We have, by \eqref{ST-z},
\[
    Q_{\tr{n/2}} = nH_{\tr{n/2}}+\sum_{k=1}^{\tr{n/2}}R\Lpa{\frac kn}
    +\frac32\,H_{\tr{n/2}}+O(1).
\]
In view of the bounded derivative of $R$ in the unit interval, 
we then deduce, by a standard application of the Euler-Maclaurin 
summation formula (approximating the sum by an integral), that 
\[
    Q_{\tr{n/2}} = n\log n + C_0\, n + \frac32\log n + O(1), 
\]
where
\[
    C_0 := \gamma-\log 2 +\int_0^{\frac12}
    \left(\frac1{S_1(t)}-\frac1t\right)\dd t
    \approx -0.69622\,72155\dots
\]
By the relation $\Delta_n(k) = \Delta_{n-1}^*(k)\lpa{1-\frac1n}^n$, 
we then deduce \eqref{S1D}, proving the theorem. \qed

See Figure~\ref{fig:differences} for the graphical rendering of 
the various approximations derived.

\renewcommand{\arraystretch}{3.0}

\begin{figure}[!ht]
\begin{tiny}
\begin{center}   
\begin{tabular}{ccc}
$\Delta_n^*(k)-S_1(\alpha)$ & 
$\Delta_n^*(k)-S_1(\alpha)-\frac{T_1(\alpha)}{n}$ & 
$\Delta_n^*(k)-S_1(\alpha)-\frac{T_1(\alpha)}{n}
-\frac{T_2(\alpha)}{n^2}$ \\
\includegraphics[height=2.2cm]{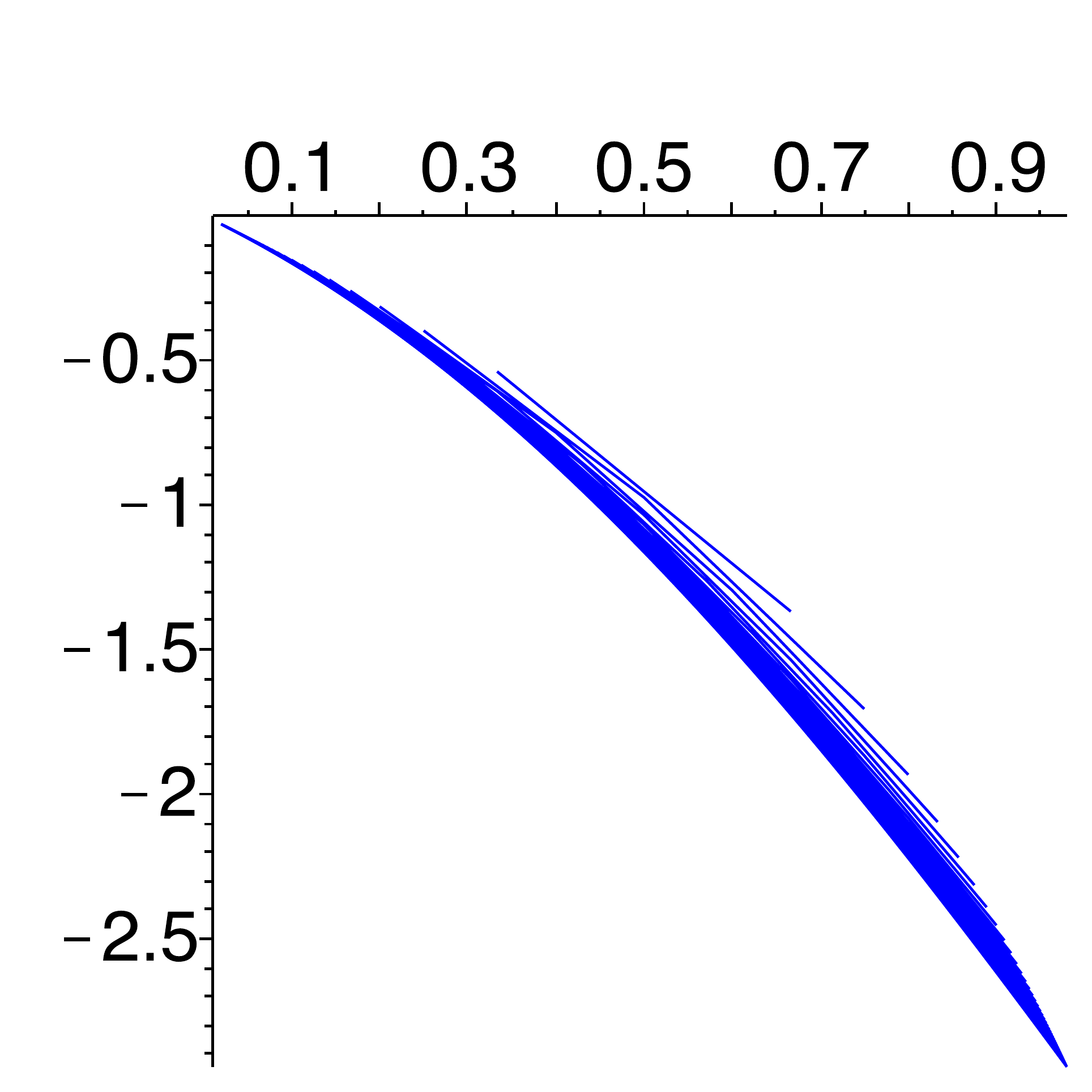} &
\includegraphics[height=2.2cm]{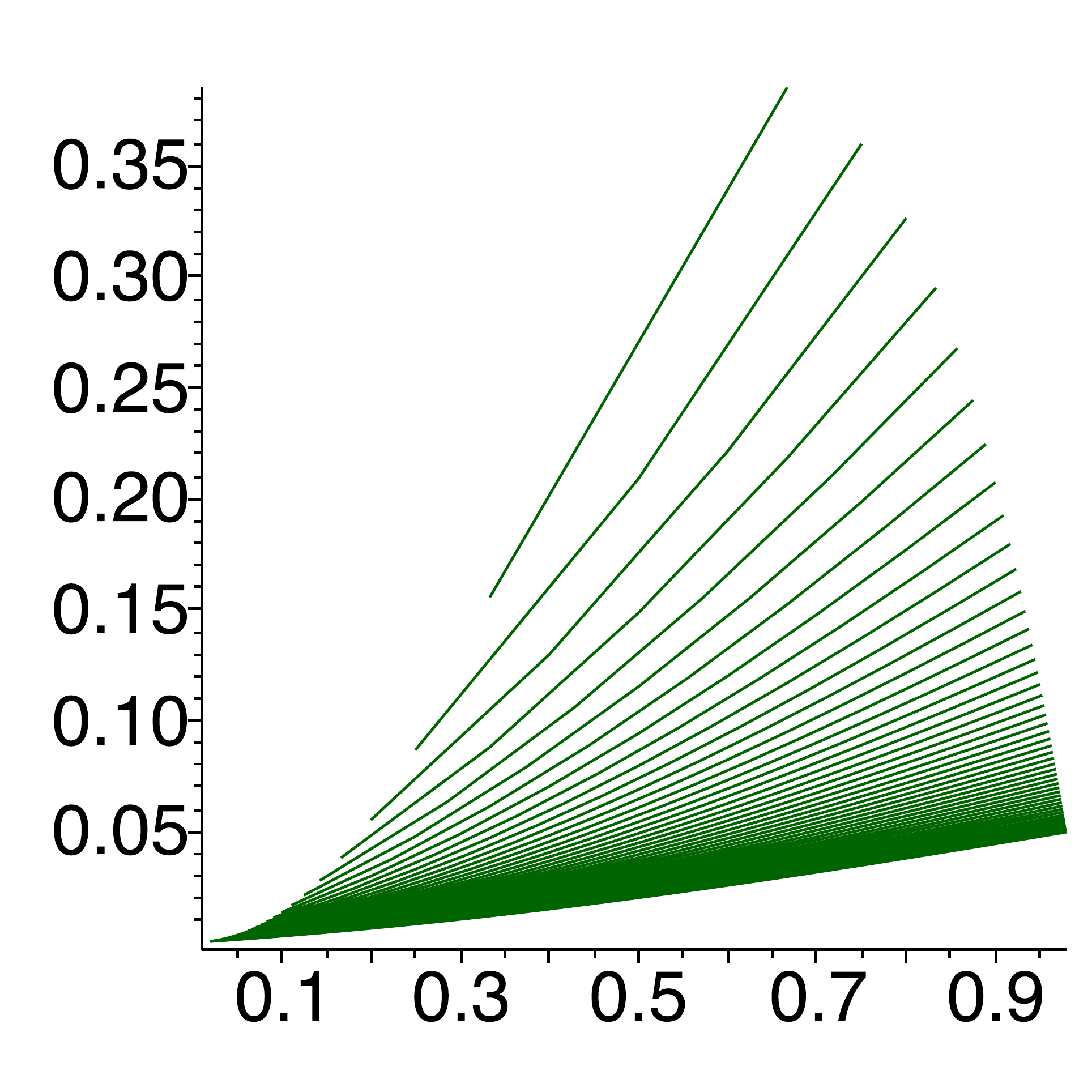} &
\includegraphics[height=2.2cm]{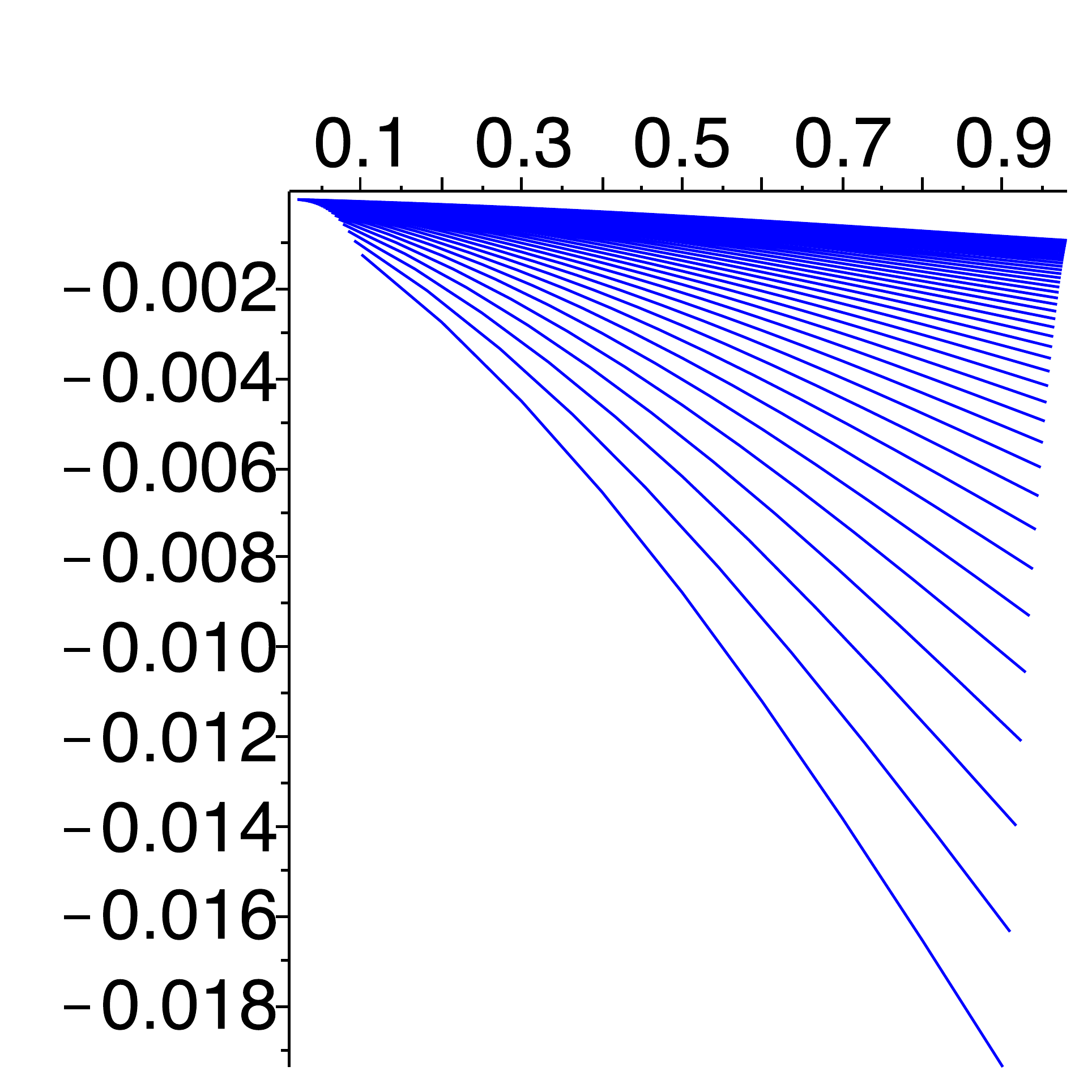} \\
$\frac1{\Delta_n^*(k)}-\frac1{S_1(\alpha)}$ & 
$\frac1{\Delta_n^*(k)}-\frac1{S_1(\alpha)}+\frac{T_1(\alpha)}
{nS_1(\alpha)^2}$ & 
$\begin{array}{l}
\frac1{\Delta_n^*(k)}-\frac1{S_1(\alpha)}
+\frac{T_1(\alpha)}{nS_1(\alpha)^2}\\+
\frac{T_1(\alpha)^2-T_2(\alpha)S_1(\alpha)}{n^2S_1(\alpha)^2}
\end{array}$\\
\includegraphics[height=2.2cm]{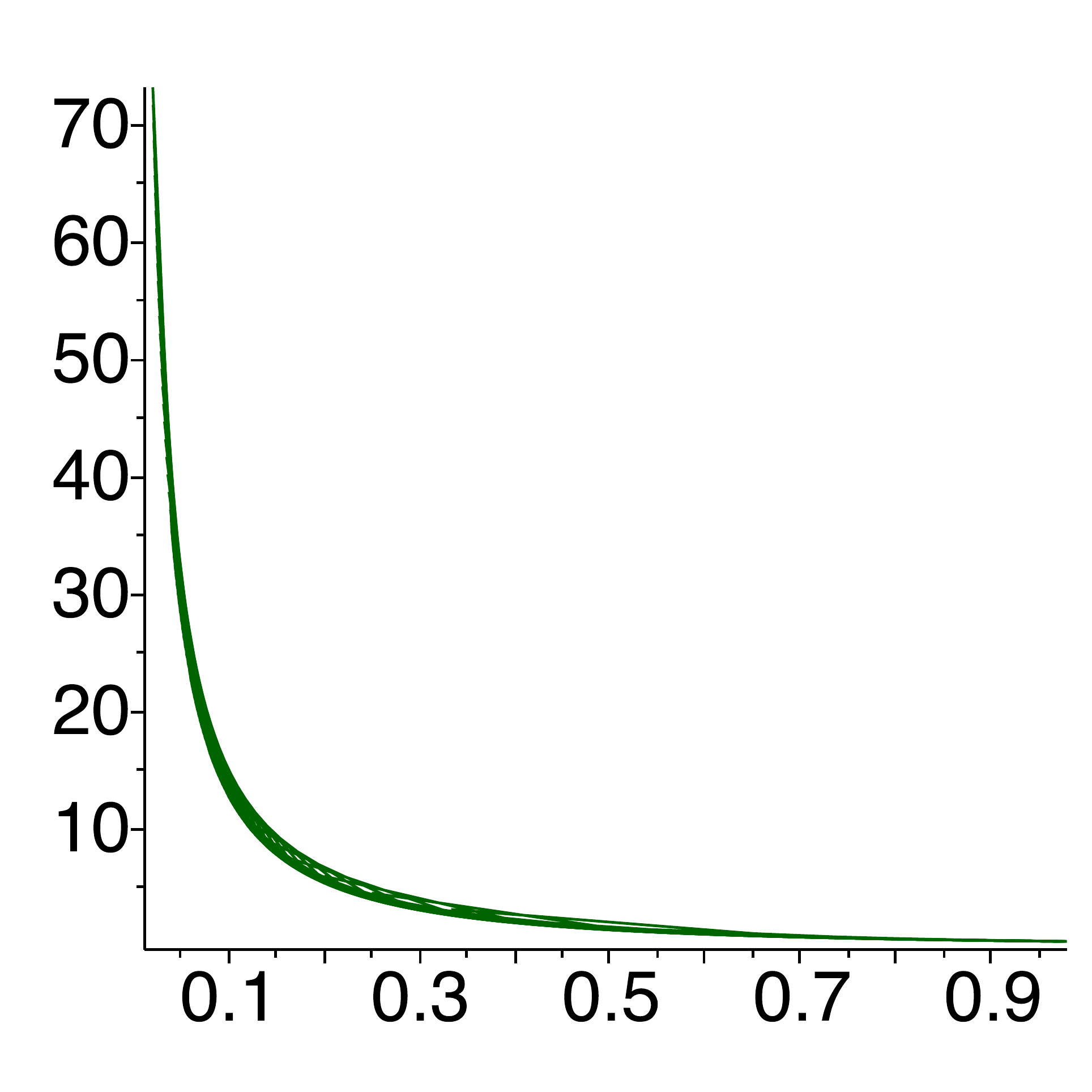} &
\includegraphics[height=2.2cm]{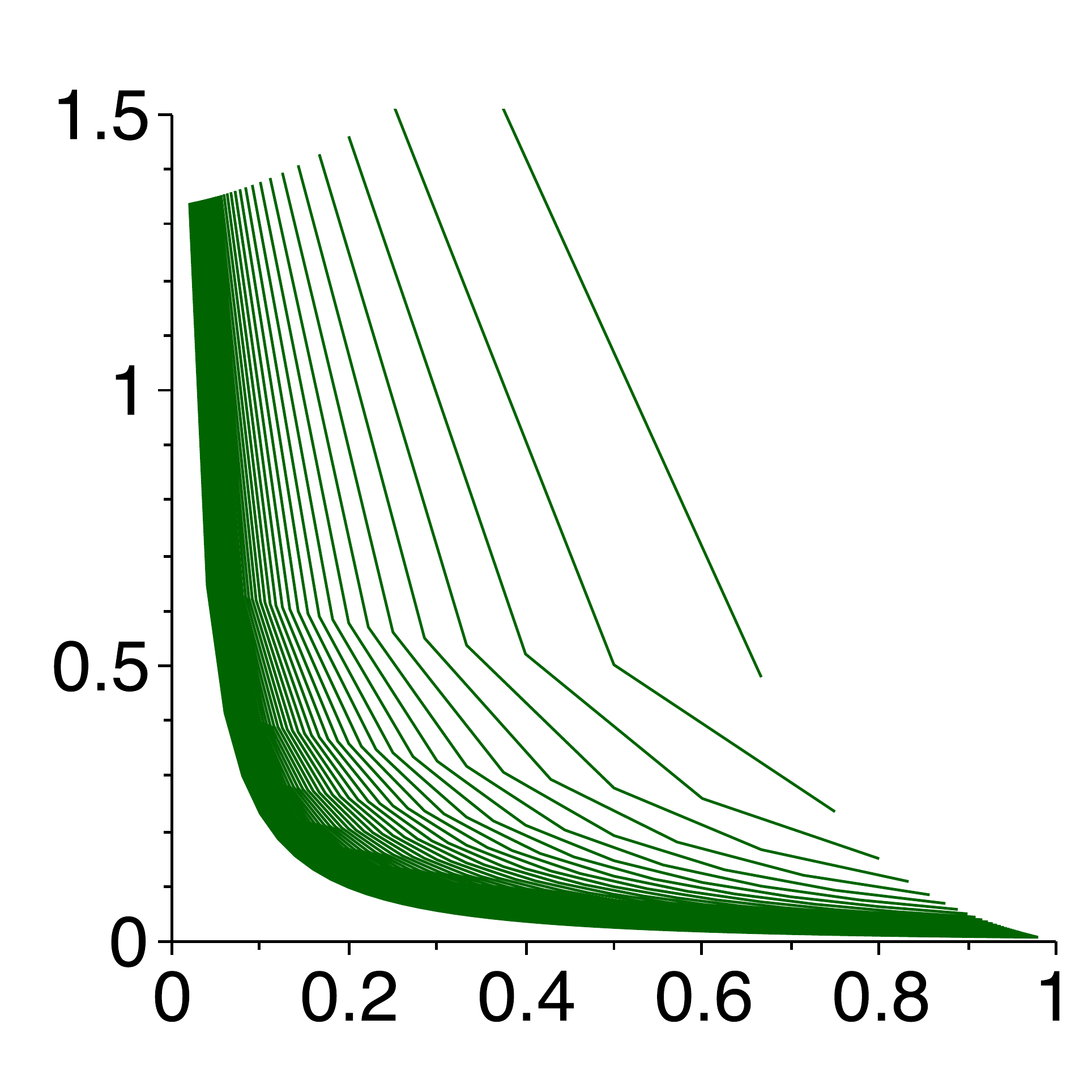} &
\includegraphics[height=2.2cm]{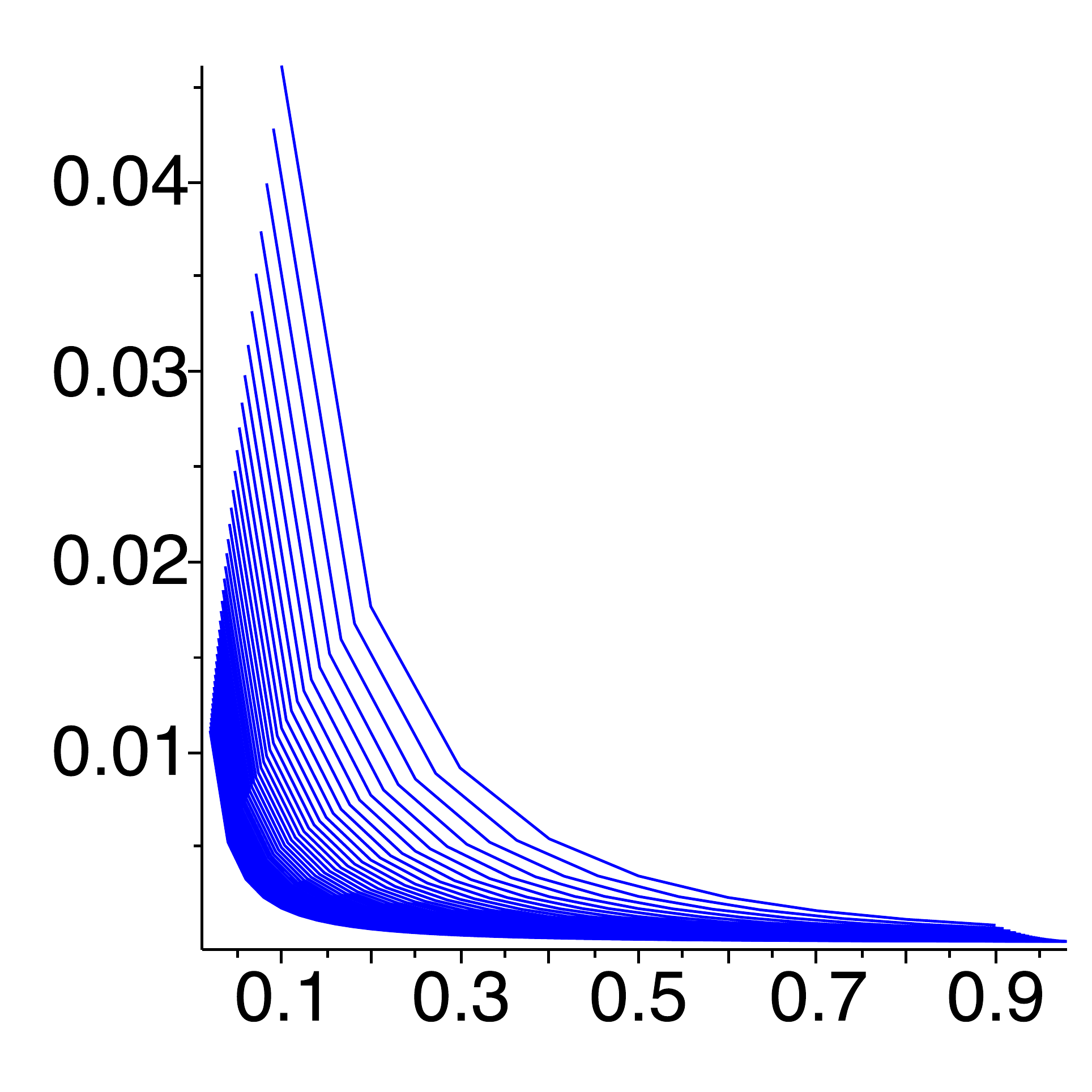}     
\end{tabular}
\end{center}
\end{tiny}
\vspace*{-.4cm}
\caption{Differences between $\Delta_n^*(k),\frac1{\Delta_n^*(k)}$ 
and their asymptotic approximations for $n=2,\dots,50$ (in increasing 
order of the density of the curves) and $k=1,\dots,n$ (normalized in 
the unit interval).}
\end{figure} 

\begin{figure}[!ht]
\begin{tiny}
\begin{center}   
\begin{tabular}{cc}
$\sum_{1\le k\le \tr{\frac n2}}\frac1{\Delta_n(k)}-\mathbf{E}(X_n)$ & 
$\sum_{1\le k\le \tr{\frac n2}}\frac1{\Delta_n(k)}-\mathbf{E}(X_n)
-\frac{e}2\log n$ \\
\includegraphics[height=2.6cm]{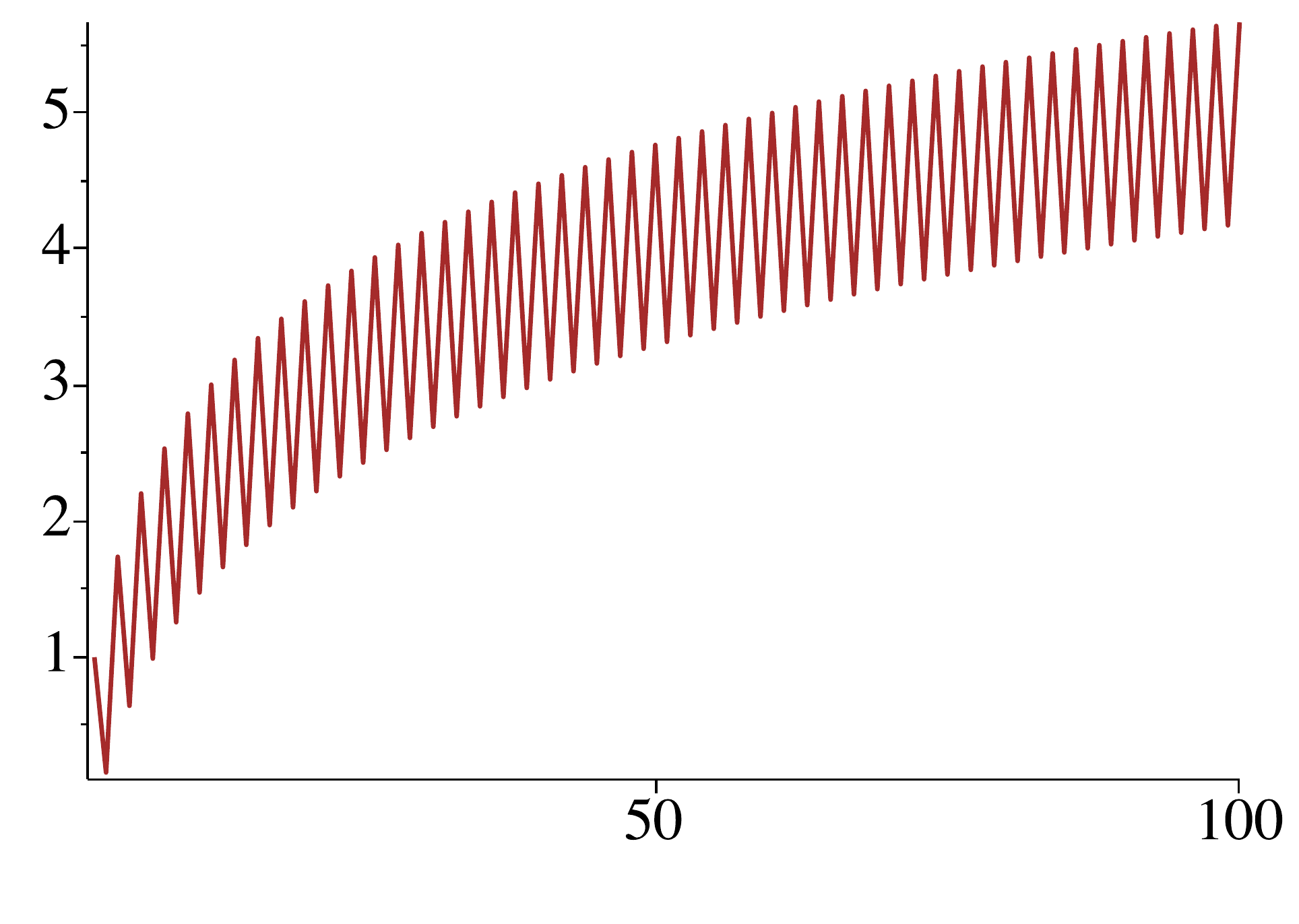} &
\includegraphics[height=2.6cm]{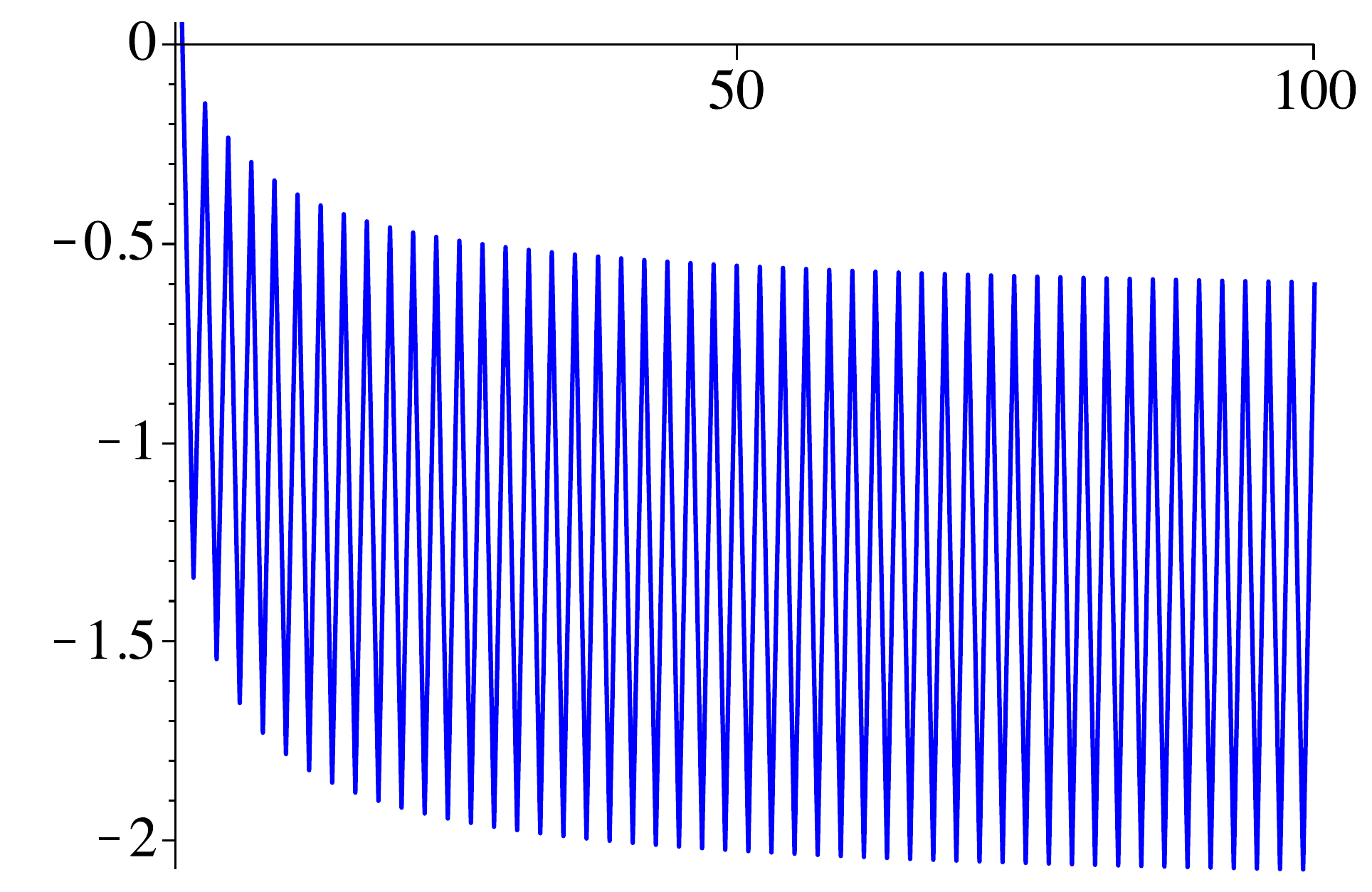}    
\end{tabular}
\end{center}
\end{tiny}
\vspace*{-.4cm}
\caption{Differences between the exact expected runtime and 
$\sum_{1\le k\le \tr{\frac n2}}\frac1{\Delta_n(k)}$ without (left) 
and with (right) the correction term $\frac e2\log n$.}
\label{fig:differences}
\end{figure}

\section*{Conclusions}
We have revisited drift analysis for the fundamental problem of bounding the 
expected runtime of the \ea on the \om problem. Using novel drift theorems 
involving error bounds, we have bounded the expected runtime when starting 
from $\tr{n/2}$ ones, up to additive terms of logarithmic order; more precisely we have
\[
\sum_{k=1}^{\tr{n/2}}\frac{1}{\Delta(k)} - \expect{T\mid X_0=\tr{n/2}} \in  [c_1\log n,  c_2\log n]
\]
for explicitly computed constants $c_1,c_2>0$. This for the first time gives an absolute error bound 
for the expected runtime. Then by standard asymptotic methods, we have 
found that $\sum_{k=1}^{\tr{n/2}}\frac{1}{\Delta(k)}$ overestimates the exact expected runtime 
by a term ${(e/2)\log n+O(1)}$.

\section*{Acknowledgements}
Partially supported by an Investigator Award from Academia Sinica 
under the Grant AS-IA-104-M03.

\end{document}